\documentclass[letterpaper]{article}
\usepackage[preprint]{aaai2027}
\usepackage[hyphens]{url}
\usepackage{graphicx}
\usepackage{natbib}
\usepackage{caption}
\usepackage{booktabs}
\usepackage{array}
\usepackage{amsfonts}
\usepackage{amsmath}
\usepackage{amssymb}
\usepackage{xspace}
\usepackage{amsthm}
\usepackage{algorithm}
\usepackage{algpseudocode}
\usepackage{longtable}

\providecommand{\pdfinfo}[1]{}
\newcommand{\method}{PATH\xspace}
\newcommand{\benchmark}{PATHBench\xspace}
\newcommand{\validthreshold}{0.885}

\newcolumntype{L}[1]{>{\raggedright\arraybackslash}p{#1}}
\newcolumntype{R}[1]{>{\raggedleft\arraybackslash}p{#1}}
\newcolumntype{C}[1]{>{\centering\arraybackslash}p{#1}}
\newtheorem{proposition}{Proposition}[section]

\title{PATH: Next-Interval Prediction via Autoregressive Tree Hierarchy\\on Tabular Data}

\author{
Pengxiang Cai\textsuperscript{\rm 1}\equalcontrib,
Wanchen Lian\textsuperscript{\rm 1}\equalcontrib,
Chenyang Liu\textsuperscript{\rm 1},
Xiaohan Li\textsuperscript{\rm 1},\\
Qingyuan Zeng\textsuperscript{\rm 1},
Jinhong Wang\textsuperscript{\rm 2},
Jintai Chen\textsuperscript{\rm 1}\corresponding
}
\affiliations{
\textsuperscript{\rm 1}The Hong Kong University of Science and Technology (Guangzhou)\\
\textsuperscript{\rm 2}Ant Group\\
jintaiCHEN@hkust-gz.edu.cn
}

\begin{document}

\maketitle

\begin{abstract}
Interval prediction aims to achieve a target coverage level while producing intervals that are as short as possible. Many conformal regression pipelines first predict an uncertainty surrogate and then convert it into an interval through calibration or selection. This separation supports coverage calibration, but post hoc rules largely determine the final interval and do not fully use the learned output distribution. We observe that the resulting intervals have inherently hierarchical geometry: an interval can be recursively refined into nested subintervals, and binary trees naturally represent this structure. We formulate this hierarchy as next-interval prediction and propose \method, which learns how probability mass flows from each interval to its next nested subintervals. \method predicts a base leaf distribution and uses an autoregressive decoder to refine branch probabilities. Matching the distribution to the interval hierarchy aligns learning with extraction: \method accumulates probability over adjacent output intervals and returns the shortest contiguous range reaching a selected mass. We compare \method with 24 baselines for interval prediction on \benchmark, comprising 56 OpenML regression datasets. \method substantially shortens the resulting intervals, achieving the lowest mean normalized length, \(0.1473\), while maintaining mean coverage of \(0.9144\). These results establish hierarchical output modeling as an effective approach for compact interval prediction on tabular data. Code is publicly available at \url{https://github.com/pxcai/PATH}.
\end{abstract}

\section{Introduction}

Interval prediction turns a tabular regressor from a point predictor into a tool for decision support. In applications such as risk assessment, demand forecasting, and scientific modeling, users need both a predicted value and an uncertainty range for possible outcomes. The central performance goal is interval length under a coverage constraint: intervals should reach the desired coverage level, and among intervals that do so, shorter intervals are more informative. Figure~\ref{fig:path-overview} provides an overview of the \method prediction pipeline.

Conformal prediction provides a general way to construct reliable sets under exchangeability. At a high level, many conformal regression methods share a common division of labor: a learned predictor or partition defines an uncertainty surrogate, and a calibration or interval selection rule converts that surrogate into an output interval. Split conformal prediction calibrates residual scores from a point predictor~\citep{vovk2005algorithmic,lei2018distributionfree,angelopoulos2023conformal}. Conformalized quantile regression calibrates lower and upper conditional quantiles~\citep{koenker1978regression,romano2019conformalized}. More recent methods improve efficiency by using conditional histograms, thresholding interquantile regions, selecting conditional interquantile intervals, localizing calibration with trees or forests, or applying conformal set construction after response discretization~\citep{sesia2021conformal,luo2025conformal,guo2026fast,cabezas2024regression,guha2023conformal}. The surrogate differs across methods, but the interval itself is usually shaped by a rule applied after training around the learned object.

\begin{figure}[t]
\centering
\includegraphics[width=\columnwidth]{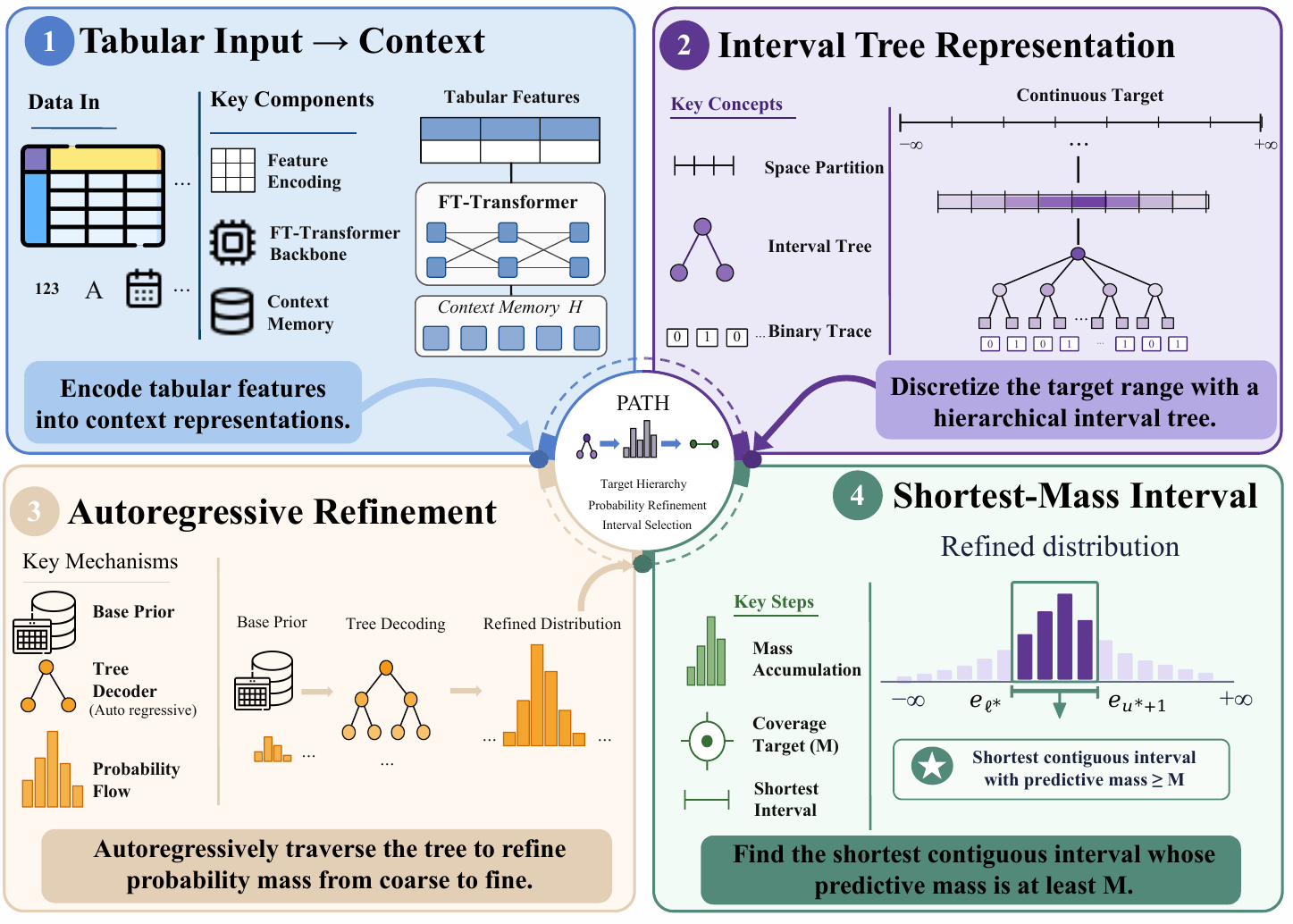}
\caption{Overview of the \method prediction pipeline.}
\label{fig:path-overview}
\end{figure}

This structure is effective for coverage, but it leaves an output-side opportunity. The final prediction object is a short contiguous segment with hierarchical geometry: any interval can be recursively divided into nested subintervals, from coarse uncertainty regions to precise boundaries. Standard surrogate objectives do not explicitly expose this hierarchy to the model. We ask whether a model can instead learn a distribution aligned with extracting short contiguous intervals.

Binary trees naturally capture this structure. A tree over the target range recursively splits coarse intervals into finer intervals, and each target value is represented by the path to its leaf. At each level, the model predicts how probability mass flows from the current interval to its next nested subintervals; we call this formulation \emph{next-interval prediction}. Coarse uncertainty and fine boundary placement therefore become part of the same distribution, directly aligning the model output with compact interval extraction.

We propose \method, a tree autoregressive model for interval prediction on tabular data. The continuous target is discretized into ordered leaves of a complete binary tree, and each training value is represented by the binary path to its leaf. A tabular transformer predicts a base distribution over leaves. An autoregressive decoder reads prefixes in the interval tree and predicts residual branch logits that refine how much mass is assigned to the left and right child branches at each node. Enumerating the tree yields a refined probability distribution over ordered output intervals. At inference time, \method sums probabilities over adjacent output intervals and returns the shortest contiguous target interval whose total mass exceeds a validation-selected threshold.

We evaluate \method against 24 baselines for interval prediction on \benchmark, a benchmark assembled from 56 OpenML regression datasets. The main comparison is repeated with 10 random seeds for each dataset. \method achieves the shortest mean normalized interval length among all compared methods while maintaining mean coverage above the nominal 0.9 target. Depth and model-design ablations further show that finer interval hierarchies and autoregressive refinement improve interval efficiency.

We make three main contributions:
\begin{itemize}
  \item We formulate interval prediction on tabular data as next-interval prediction: recursive coarse-to-fine refinements are represented by a binary interval tree, matching the learned distribution to the compact contiguous intervals extracted at inference time.
  \item We propose \method, a tree-autoregressive method that encodes continuous targets as binary traces, predicts a base leaf distribution, refines branch probabilities with autoregressive residual logits, and constructs intervals by accumulating probability mass over adjacent output intervals.
  \item We provide a systematic evaluation of tree-structured interval modeling on \benchmark against 24 interval baselines, repeating the main comparison with 10 random seeds for each dataset. \method achieves the shortest mean normalized interval length among all compared methods while maintaining empirical mean coverage above the nominal \(0.9\) target, and the ablations measure depth scaling, validation-target behavior, and the contribution of autoregressive refinement.
\end{itemize}

\section{Related Work}

\paragraph{Conformal and adaptive intervals.}
Split conformal prediction calibrates held-out residuals under exchangeability, while conformalized quantile regression calibrates conditional quantile endpoints~\citep{vovk2005algorithmic,lei2018distributionfree,angelopoulos2023conformal,koenker1978regression,romano2019conformalized}. Both pair naturally with XGBoost, LightGBM, CatBoost, random forests, and quantile regression forests~\citep{chen2016xgboostb,ke2017lightgbma,prokhorenkova2018catboostb,breiman2001random,meinshausen2006quantile}. More adaptive methods estimate conditional histograms, threshold candidate regions, select conditional interquantile intervals, localize calibration with trees or forests, or discretize regression as classification~\citep{sesia2021conformal,luo2025conformal,guo2026fast,cabezas2024regression,guha2023conformal}. These approaches motivate our baseline suite and establish localization and response discretization as effective routes to shorter intervals.

\paragraph{Neural tabular and structured-output models.}
Deep tabular models improve representations through feature selection, contextual embeddings, differentiable trees, feature tokenization, and pretrained language models~\citep{arik2021tabnetb,huang2020tabtransformera,popov2019neuralb,somepalli2021saintb,gorishniy2021revisitingb,yan2024making}. Neural uncertainty methods use ensembles, direct interval losses, evidential learning, mixture densities, or distributional regression~\citep{lakshminarayanan2017simple,pearce2018highquality,amini2020deep,bishop1994mixture,rugamer2023deepregression}. Structured regression outputs include ordinal labels, CDF targets, and predictive distributions~\citep{cao2020rank,fu2018deep,li2021deep,izbicki2021cdsplit,vovk2017nonparametric}. \method builds on this output-side perspective by organizing probability over nested target intervals, refining tree-prefix probabilities, and extracting adjacent mass directly as the final interval.

\begin{figure*}[!t]
\centering
\includegraphics[width=\textwidth]{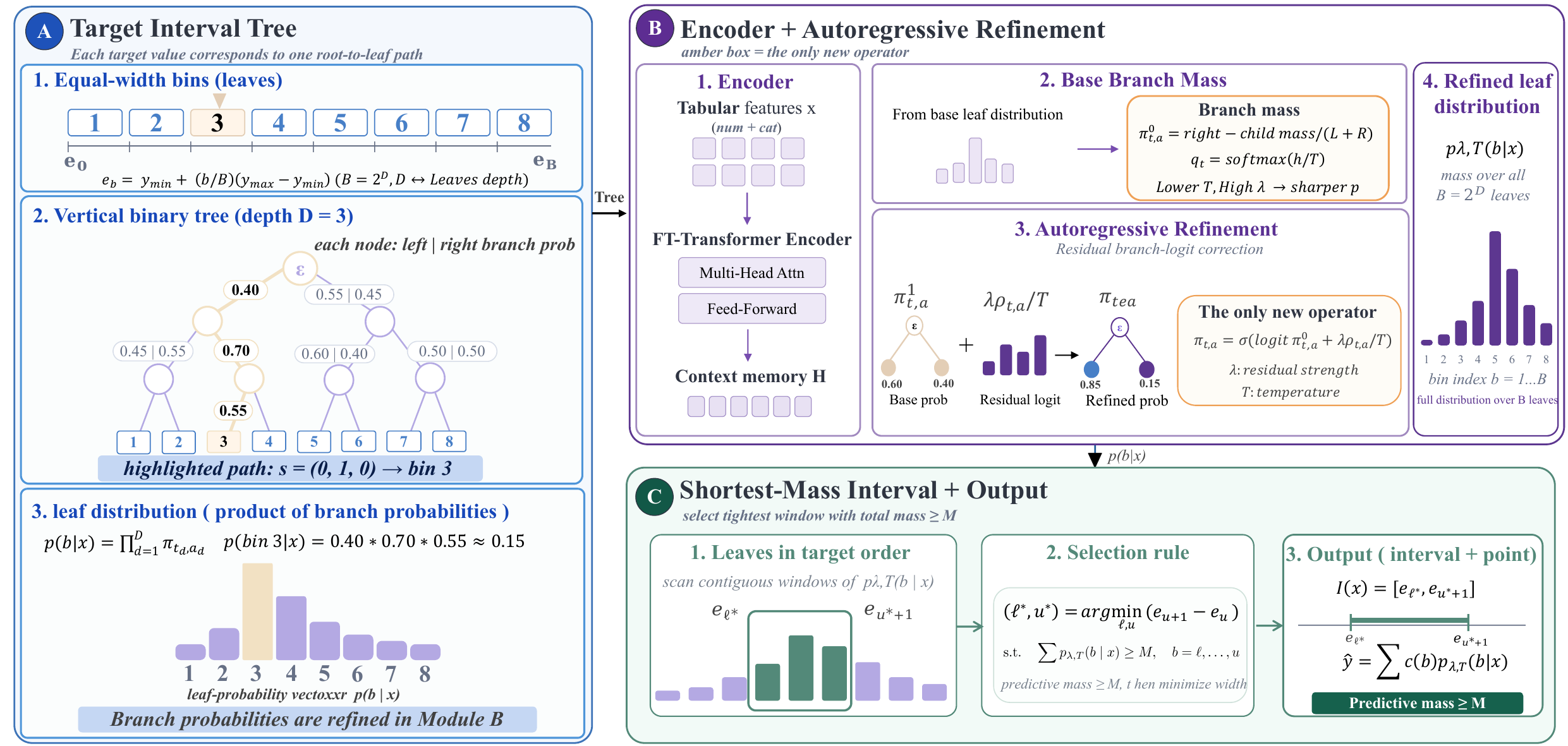}
\caption{\method architecture. A binary interval tree defines ordered leaves, an FT-Transformer predicts base probabilities, autoregressive residual logits refine branches, and adjacent probability mass yields the shortest selected interval.}
\label{fig:path-method}
\end{figure*}

\section{Method}

\method learns a structured distribution in the output space, refines it with an autoregressive interval hierarchy, and extracts compact intervals by accumulating probability mass over adjacent output intervals. Figure~\ref{fig:path-method} illustrates the complete \method architecture. The construction has three parts. First, the training target range is organized as a complete binary tree whose leaves are ordered target intervals. Second, a tabular encoder predicts a base distribution over these leaves, and an autoregressive decoder refines branch probabilities associated with tree prefixes. Third, validation data select a small set of inference parameters, after which prediction accumulates probability mass over adjacent output intervals and returns the shortest contiguous interval whose total mass reaches the selected threshold.

\subsection{Target Tree Encoding}

Let \((x_i,y_i)\) be a tabular regression dataset with scalar target \(y_i\). For depth \(D\), we define \(B=2^D\) ordered leaf intervals by fitting boundaries on the training target range:
\begin{equation}
\label{eq:leaf-boundaries}
  e_b = y_{\min}^{\mathrm{tr}} + \frac{b}{B}
  \left(y_{\max}^{\mathrm{tr}}-y_{\min}^{\mathrm{tr}}\right),
  \qquad b=0,\ldots,B .
\end{equation}
Uniform boundaries are used because the evaluation target is interval length in the original target units. With equal-width leaves, a contiguous set of leaves has a direct length interpretation and each tree level corresponds to a geometric refinement of the target range. Quantile boundaries are useful when class balance is the main objective, but their nonuniform widths separate class frequency from interval length. The benchmark therefore uses uniform boundaries to align target resolution with interval efficiency.

Each target is assigned to a clipped leaf index
\begin{equation}
\label{eq:leaf-index}
  z_i =
  \min\left\{B-1,\,
  \max\left(0,\,
  \left\lfloor
  \frac{B(y_i-y_{\min}^{\mathrm{tr}})}
  {y_{\max}^{\mathrm{tr}}-y_{\min}^{\mathrm{tr}}}
  \right\rfloor
  \right)\right\}.
\end{equation}
The binary expansion of \(z_i\) gives a trace \(s_i=(s_{i,1},\ldots,s_{i,D})\). For a prefix \(a\in\{0,\ldots,2^t-1\}\) at depth \(t\), the corresponding target region covers the leaves
\begin{equation}
\label{eq:tree-cell}
  \mathcal{C}_t(a)
  =
  \{b:\, a2^{D-t} \le b < (a+1)2^{D-t}\}.
\end{equation}
The trace records which child interval contains the target at every refinement step. This representation preserves the scalar target ordering while exposing coarse-to-fine interval structure to the model.

\subsection{Base Distribution and Autoregressive Residuals}

The tabular encoder maps features to a memory sequence \(H=f_\theta(x)\). Numeric features are tokenized with affine embeddings for each feature, categorical features use learned embeddings, and a CLS token summarizes the row. Our default implementation follows the FT-Transformer design~\citep{gorishniy2021revisitingb}. The output side then combines a direct leaf distribution with autoregressive corrections defined on the interval tree.

The direct leaf head produces base logits \(h^{\mathrm{leaf}}(x)\in\mathbb{R}^{B}\). At inference temperature \(T\), the base target distribution is
\begin{equation}
\label{eq:base-leaf-distribution}
  q_T(b\mid x)=
  \frac{\exp(h^{\mathrm{leaf}}_b(x)/T)}
  {\sum_{b'=0}^{B-1}\exp(h^{\mathrm{leaf}}_{b'}(x)/T)} .
\end{equation}
For each tree node \(\mathcal{C}_t(a)\), let \(L_t(a)\) and \(R_t(a)\) be its left and right child leaf sets. The base probability of the right branch is
\begin{equation}
\label{eq:base-branch-probability}
  \pi^{0}_{t,a}(x)
  =
  \frac{\sum_{b\in R_t(a)}q_T(b\mid x)}
  {\sum_{b\in L_t(a)\cup R_t(a)}q_T(b\mid x)} .
\end{equation}
The decoder receives \([\mathrm{SOS}]\) followed by the prefix bits of node \(a\) and attends to \(H\). It outputs a residual branch logit \(\rho_{t,a}(x)\), which corrects the base branch probability:
\begin{equation}
\label{eq:residual-branch-probability}
  \pi^{\lambda}_{t,a}(x)
  =
  \sigma\left(
  \mathrm{logit}\left(\pi^{0}_{t,a}(x)\right)
  + \lambda\,\rho_{t,a}(x)/T
  \right),
\end{equation}
where \(\lambda\) controls the strength of the autoregressive correction. Enumerating every tree prefix yields the refined leaf distribution
\begin{equation}
\label{eq:refined-leaf-distribution}
\begin{aligned}
  p_{\lambda,T}(b\mid x)
  ={}& \prod_{t=0}^{D-1}
  \left(\pi^{\lambda}_{t,a_t(b)}(x)\right)^{s_{t+1}(b)} \\
  &\times
  \left(1-\pi^{\lambda}_{t,a_t(b)}(x)\right)^{1-s_{t+1}(b)} .
\end{aligned}
\end{equation}
When \(\lambda=0\), the model falls back to the direct leaf distribution up to numerical clamping. The decoder is therefore an autoregressive distributional correction rather than a greedy path generator, and all leaves remain available when the final interval is extracted.

\subsection{Training Objective}

During training, temperature scaling is disabled and the residual branch uses \(\lambda_{\mathrm{train}}=1\); temperature is selected only at evaluation time. Let \(q\) denote the direct leaf distribution and \(p\) the autoregressively refined distribution. We supervise both distributions with three losses defined over output intervals. The leaf loss is standard cross entropy:
\begin{equation}
\label{eq:leaf-loss}
  \mathcal{L}_{\mathrm{leaf}}(r)
  =
  -\frac{1}{n}\sum_{i=1}^{n}\log r(z_i\mid x_i),
  \qquad r\in\{q,p\}.
\end{equation}
The prefix loss encourages probability mass to remain inside the true region at each level of refinement:
\begin{equation}
\label{eq:prefix-loss}
  \mathcal{L}_{\mathrm{prefix}}(r)
  =
  -\frac{1}{n(D-1)}
  \sum_{i=1}^{n}\sum_{t=1}^{D-1}
  \log
  \sum_{b\in \mathcal{C}_t(a_{i,t})}
  r(b\mid x_i).
\end{equation}
The CDF loss supervises the ordered cumulative distribution using binary labels over thresholds:
\begin{equation}
\label{eq:cdf-loss}
\begin{aligned}
  \mathcal{L}_{\mathrm{cdf}}(r)
  ={}& -\frac{1}{n(B-1)}
  \sum_{i=1}^{n}\sum_{k=0}^{B-2}
  \Bigl[
  \mathbf{1}[k\ge z_i]\log F_r(k\mid x_i) \\
  &\qquad + \mathbf{1}[k<z_i]
  \log\!\left(1-F_r(k\mid x_i)\right)
  \Bigr],
\end{aligned}
\end{equation}
where \(F_r(k\mid x)=\sum_{b\le k}r(b\mid x)\). The final objective is
\begin{equation}
\label{eq:training-objective}
\begin{aligned}
  \mathcal{L} ={}&
  0.5\Bigl(
  \mathcal{L}_{\mathrm{leaf}}(q)
  +\mathcal{L}_{\mathrm{prefix}}(q)
  +\mathcal{L}_{\mathrm{cdf}}(q)
  \Bigr) \\
  &+\mathcal{L}_{\mathrm{leaf}}(p)
  +\mathcal{L}_{\mathrm{prefix}}(p)
  +\mathcal{L}_{\mathrm{cdf}}(p) \\
  &+0.01\,\mathcal{L}_{\mathrm{res}}.
\end{aligned}
\end{equation}
where \(\mathcal{L}_{\mathrm{res}}\) is a squared penalty on residual logits, weighted by node mass. The direct leaf terms anchor the global distribution, while the autoregressive terms train the refined tree distribution used at inference time.

\subsection{Interval Extraction and Validation Selection}

For a test point, \method obtains \(p_{\lambda,T}(b\mid x)\), a probability distribution over ordered output intervals. The point prediction is the posterior mean over leaf centers. For interval prediction, given mass \(M\), the model sums probabilities over adjacent output intervals and selects the shortest contiguous target interval with total probability mass at least \(M\):
\begin{equation}
\label{eq:shortest-mass-interval}
\begin{aligned}
  (\ell^*,u^*) ={}&
  \arg\min_{0\le \ell\le u<B} (e_{u+1}-e_{\ell}) \\
  &\mathrm{s.t.}\quad
  \sum_{b=\ell}^{u}p_{\lambda,T}(b\mid x)\ge M .
\end{aligned}
\end{equation}
The interval is \(I(x)=[e_{\ell^*},e_{u^*+1}]\).

The inference parameters \(T\), \(M\), and \(\lambda\) are selected on the validation split. For a selection target \(\eta\), we enumerate a fixed grid and first retain candidates whose validation coverage is at least \(\eta\). Among those candidates, we choose the one with the shortest validation interval. If no candidate reaches \(\eta\), we choose the candidate closest to the target coverage and use length to break ties. The main setting uses \(\eta=0.905\), and the selected interval is then evaluated on the test split.

Appendix B summarizes the full pipeline in pseudocode, with each operation tied to the corresponding equation above. Appendix A records basic properties of this construction: the refined tree distribution is normalized, \(\lambda=0\) recovers the direct leaf distribution, shortest-mass extraction is the exact optimizer over contiguous leaf intervals, the discretization error of grid-aligned interval extraction decreases with depth, and finite-grid validation selection admits a standard concentration bound.

\section{Experiments}
\label{sec:experiments}

The experiments evaluate four aspects of \method. We first compare its interval compactness with 24 baselines under matched empirical coverage requirements. We then test whether the advantage persists at the nominal \(0.900\) coverage level and across repeated data splits. Finally, we examine scaling with tree depth and isolate the contributions of structured output modeling and autoregressive refinement.

\subsection{Experimental Protocol}

\paragraph{Datasets and splits.}
\benchmark comprises 56 OpenML regression datasets~\citep{vanschoren2014openml}. Figure~\ref{fig:pathbench-overview} summarizes its variation in sample size, input dimensionality, and feature composition. The main comparison uses random seeds \(0,\ldots,9\), each defining an independent 60/20/20 train/validation/test split, producing 560 evaluations per method. Preprocessing parameters and target boundaries are fitted independently on each training split. Numeric features are standardized using training statistics, while categorical features use ordinal encoding with a reserved unknown category. The model-design ablation uses the same 10 random seeds. Depth, validation-target, and distributional analyses use the representative split unless stated otherwise.

\begin{figure}[t]
\centering
\includegraphics[width=\columnwidth]{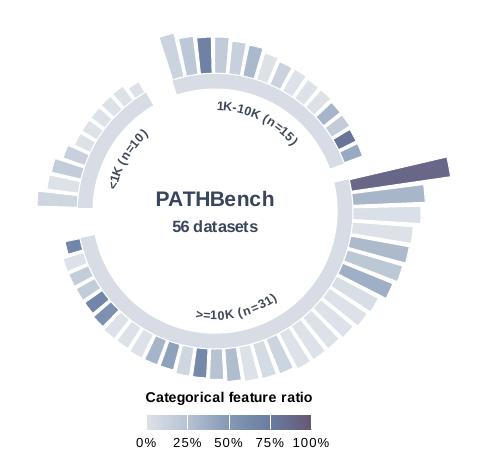}
\caption{Composition of \benchmark. Arc span shows dataset count by sample-size group; radial-bar length and color encode feature count (square-root scale) and categorical-feature proportion.}
\label{fig:pathbench-overview}
\end{figure}

\paragraph{\method configuration.}
The main configuration uses \(D=8\), yielding 256 ordered leaf intervals with uniform boundaries over the training target range. The encoder is an FT-Transformer with \(d_{\mathrm{model}}=64\), four attention heads, two layers, dropout \(0.1\), and ReGLU feed-forward layers. We train for 60 epochs using AdamW with learning rate \(10^{-3}\) and weight decay \(10^{-4}\). Validation loss under teacher forcing selects the checkpoint. For interval selection, we evaluate the following inference grid on the validation set:
{\small
\[
\begin{aligned}
M &\in \{0.40,0.45,\ldots,0.90,0.92,0.95\},\\
T &\in \{0.7,0.8,0.9,1.0,1.2,1.5\},\\
\lambda &\in \{0,0.25,0.5,0.75,1.0\}.
\end{aligned}
\]
}
Here \(M\) denotes the accumulated probability mass, \(T\) controls inference-time distribution sharpness, and \(\lambda\) scales the autoregressive residual. Among settings reaching validation coverage \(\eta\), we select the one with the shortest mean interval. The main setting uses \(\eta=0.905\). The scaling analysis additionally considers \(\eta\in\{0.905,0.910,0.915\}\) and \(D\in\{4,5,6,7,8\}\).

\paragraph{Baselines.}
We compare \method with 24 baselines for interval prediction. The suite includes split conformal prediction~\citep{vovk2005algorithmic,lei2018distributionfree,angelopoulos2023conformal} using CatBoost, XGBoost, LightGBM, random forests, and MLPs~\citep{prokhorenkova2018catboostb,chen2016xgboostb,ke2017lightgbma,breiman2001random}. It also includes conformalized quantile regression~\citep{koenker1978regression,romano2019conformalized} using boosted trees, quantile regression forests~\citep{meinshausen2006quantile}, and MLPs. Further comparisons cover QRF- and MLP-based CHR, CTI, CIR, and CIR+~\citep{sesia2021conformal,luo2025conformal,guo2026fast}, MLP-R2CCP~\citep{guha2023conformal}, RF-TreeQuantile-CQR, and LOCART-family methods~\citep{cabezas2024regression}. Every method uses identical random seeds and evaluation metrics.

\paragraph{Metrics and comparison rule.}
We report empirical coverage, mean and median interval length normalized by the training target range, mean rank, and wins. A method is valid on an evaluation when test coverage reaches \(\validthreshold\), allowing a \(0.015\) tolerance below the nominal \(0.900\) target. Valid methods are ranked by normalized length. Remaining methods follow by coverage error and then length. This rule measures compactness only among methods satisfying the empirical coverage requirement.

\subsection{Main Comparison}

Table~\ref{tab:cross-dataset-summary} summarizes the 560-evaluation main comparison. \method reaches mean coverage \(0.9144\), above the nominal \(0.900\) target. It also achieves the shortest mean and median normalized lengths, \(0.1473\) and \(0.0547\), respectively. These intervals translate into the best mean rank, \(5.25\), and the largest win count, 163. QRF-CTI ranks second with 130 wins, but its mean normalized length is \(0.2017\), \(36.9\%\) longer than that of \method. Thus, the repeated-split comparison establishes a substantial interval-length advantage at comparable mean coverage.

\begin{table}[!ht]
\centering
\caption{Main 10-seed comparison on \benchmark against 24 baselines. \method uses \(D=8\) and \(\eta=0.905\); ranking requires coverage \(\geq0.885\) before interval length. SplitCP: split conformal prediction; TQ: TreeQuantile.}
\label{tab:cross-dataset-summary}
\scriptsize
\setlength{\tabcolsep}{0.8pt}
\renewcommand{\arraystretch}{0.92}
\begin{tabular}{@{}L{0.31\columnwidth}C{0.11\columnwidth}C{0.16\columnwidth}C{0.16\columnwidth}C{0.11\columnwidth}C{0.09\columnwidth}@{}}
\toprule
Method & Cov. & \shortstack{Mean\\len/rng \(\downarrow\)} & \shortstack{Median\\len/rng \(\downarrow\)} & Rank \(\downarrow\) & Wins \(\uparrow\) \\
\midrule
\textbf{PATH} & 0.9144 & \textbf{0.1473} & \textbf{0.0547} & \textbf{5.25} & \textbf{163} \\
QRF-CTI & 0.9201 & 0.2017 & 0.1215 & 8.38 & 130 \\
QRF-CHR & 0.9046 & 0.2110 & 0.0856 & 9.03 & 12 \\
QRF-CIR & 0.8995 & 0.1800 & 0.0741 & 9.45 & 9 \\
XGBoost-CQR & 0.9047 & 0.2120 & 0.0937 & 9.96 & 6 \\
CatBoost-CQR & 0.9038 & 0.1786 & 0.0817 & 10.03 & 9 \\
LightGBM-CQR & 0.9059 & 0.2002 & 0.0833 & 10.29 & 7 \\
QRF-CIR+ & 0.8951 & 0.1760 & 0.0738 & 10.73 & 32 \\
CatBoost-SplitCP & 0.9034 & 0.1967 & 0.1298 & 11.94 & 26 \\
MLP-CTI & 0.9009 & 0.2900 & 0.2041 & 12.67 & 49 \\
XGBoost-SplitCP & 0.9035 & 0.2162 & 0.1372 & 12.77 & 27 \\
QRF-CQR & 0.9222 & 0.2159 & 0.0963 & 12.99 & 0 \\
LightGBM-SplitCP & 0.9042 & 0.2242 & 0.1407 & 13.91 & 14 \\
RF-SplitCP & 0.9043 & 0.2277 & 0.1549 & 14.08 & 2 \\
MLP-R2CCP & 0.8586 & 0.1953 & 0.1009 & 14.51 & 40 \\
RF-A-LOCART & 0.9043 & 0.2649 & 0.1664 & 14.65 & 4 \\
RF-TQ-CQR & 0.9241 & 0.2431 & 0.1295 & 14.83 & 2 \\
MLP-CHR & 0.8750 & 0.2533 & 0.1702 & 14.92 & 12 \\
RF-A-LOFOREST & 0.9001 & 0.2668 & 0.1738 & 15.32 & 0 \\
MLP-CIR+ & 0.8643 & 0.2626 & 0.1170 & 15.48 & 12 \\
RF-LOCART & 0.9090 & 0.2682 & 0.1683 & 15.78 & 0 \\
MLP-CIR & 0.9149 & 0.3330 & 0.1964 & 15.90 & 1 \\
MLP-SplitCP & 0.9031 & 0.3869 & 0.1838 & 16.24 & 2 \\
RF-LOFOREST & 0.8988 & 0.2700 & 0.1741 & 16.80 & 0 \\
MLP-CQR & 0.9037 & 0.3646 & 0.2334 & 19.09 & 1 \\
\bottomrule
\end{tabular}
\end{table}

\begin{table}[t]
\centering
\caption{Strict \(0.900\)-coverage comparison over 560 evaluations. Valid methods rank by normalized length; the remainder rank by coverage error and length. Wins are first-place finishes.}
\label{tab:coverage-attainment}
\scriptsize
\setlength{\tabcolsep}{1.8pt}
\renewcommand{\arraystretch}{0.95}
\begin{tabular}{@{}L{0.29\columnwidth}C{0.29\columnwidth}C{0.18\columnwidth}C{0.18\columnwidth}@{}}
\toprule
Method & Cov. \(\ge .900\) (\%) \(\uparrow\) & Rank@.900 \(\downarrow\) & Wins@.900 \(\uparrow\) \\
\midrule
\textbf{\method} & \textbf{76.6} & \textbf{6.46} & \textbf{177} \\
QRF-CTI & 70.7 & 9.36 & 121 \\
QRF-CQR & 74.5 & 11.45 & 11 \\
LightGBM-CQR & 62.3 & 11.77 & 14 \\
CatBoost-CQR & 55.9 & 12.37 & 14 \\
QRF-CIR & 48.8 & 13.35 & 13 \\
QRF-CIR+ & 38.6 & 15.62 & 16 \\
\bottomrule
\end{tabular}
\end{table}

Table~\ref{tab:coverage-attainment} reports \method and six strong baselines under the nominal \(0.900\) coverage requirement. \method attains this level on \(76.6\%\) of evaluations, the highest rate among the listed methods. Across the complete ranking, it also records the best mean rank, \(6.46\), and the most wins, 177. The compactness advantage therefore remains when intervals must meet the nominal target directly. Appendix E.5 provides complete counts and metrics.

Figure~\ref{fig:performance-profile} measures how often each method approaches the shortest valid interval. For each evaluation, \(L^*\) is the minimum length among methods reaching \(0.900\) coverage. A method succeeds at tolerance \(\delta\) when its length is at most \((1+\delta)L^*\). \method has the highest success fraction throughout the profile. It lies within \(10\%\) and \(25\%\) of \(L^*\) on \(47.1\%\) and \(53.6\%\) of evaluations. The corresponding QRF-CTI rates are \(36.4\%\) and \(43.9\%\). The advantage therefore reflects frequent near-optimal performance, not only isolated wins.

\begin{figure}[t]
\centering
\includegraphics[width=\columnwidth]{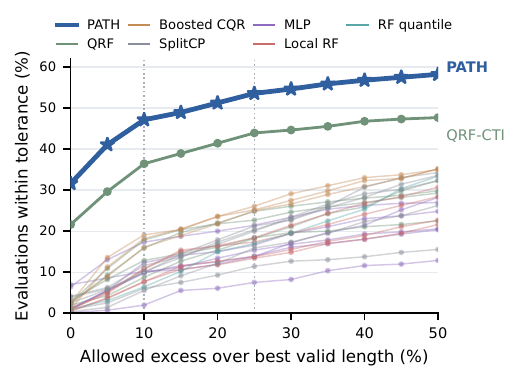}
\caption{Coverage-constrained profiles over 560 evaluations. Success at tolerance \(\delta\) requires coverage \(\geq0.900\) and length \(\leq(1+\delta)L^*\), where \(L^*\) is the shortest valid interval.}
\label{fig:performance-profile}
\end{figure}

Figure~\ref{fig:pairwise-length-advantage} quantifies both the magnitude and prevalence of the gain against six strong baselines. We first average normalized length over 10 seeds within each dataset, then compute paired baseline-minus-\method differences. Aggregate reductions range from \(16.3\%\) to \(31.8\%\), and \method is shorter on 36 to 54 of 56 datasets. All \(95\%\) dataset-bootstrap intervals remain above zero. The interval-length improvement is therefore broadly distributed across datasets, including the comparison with QRF-CTI.

\begin{figure}[t]
\centering
\includegraphics[width=\columnwidth]{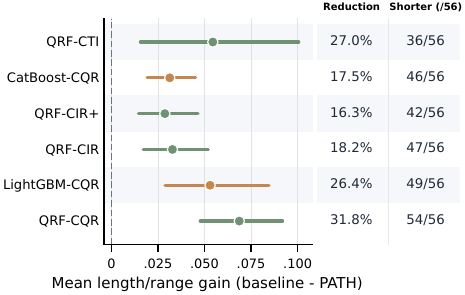}
\caption{Paired interval-length gains over six strong baselines. Dataset means average 10 seeds; points and lines show baseline-minus-\method gains and 95\% bootstrap intervals from 10,000 resamples. Positive values favor \method; right columns report overall reduction and positive-gain datasets.}
\label{fig:pairwise-length-advantage}
\end{figure}

\begin{figure}[t]
\centering
\includegraphics[width=\columnwidth]{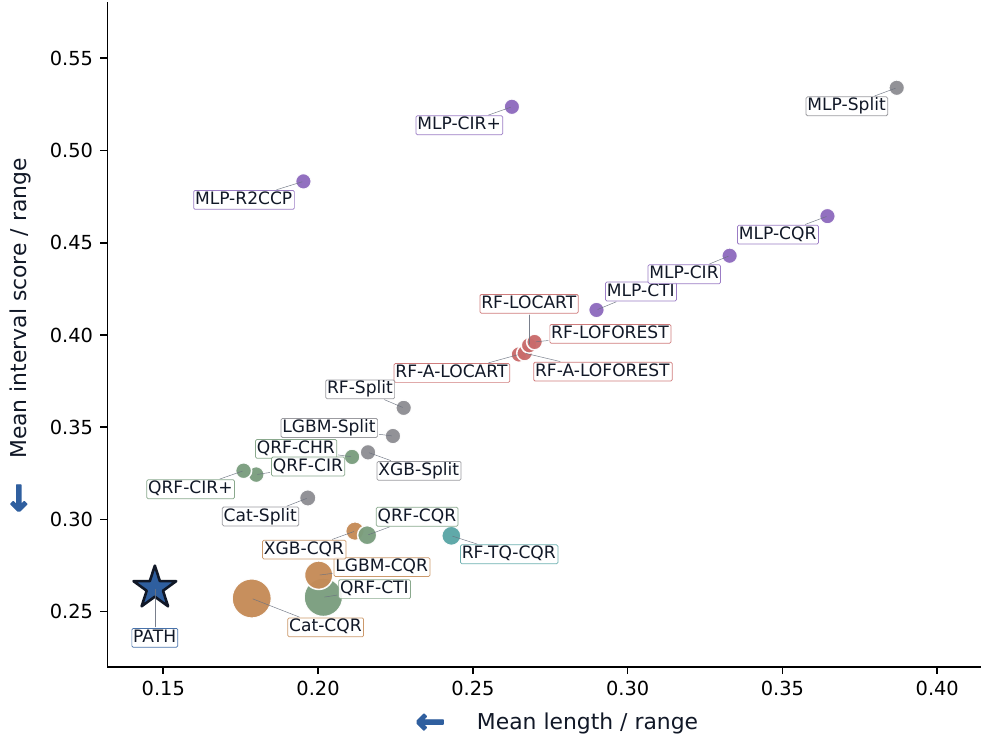}
\caption{Interval score versus normalized length over 10 splits; lower-left is better. Colors denote method families, and larger markers indicate lower scores. MLP-CHR is omitted as an outlier (length/range \(=0.2533\), score/range \(=0.8616\)).}
\label{fig:method-pareto-current}
\end{figure}

Figure~\ref{fig:method-pareto-current} examines whether compactness is obtained at the expense of overall interval score. \method occupies the shortest-length end of the lower-left performance frontier. CatBoost-CQR and QRF-CTI achieve slightly lower scores, but require substantially longer mean intervals. The plot therefore identifies \method's distinguishing advantage as interval compactness under the empirical coverage requirement. Appendix D reports all interval-score values.

\subsection{Depth and Validation Target Analysis}

Tree depth \(D\) determines output resolution, while validation target \(\eta\) controls the coverage level used for interval selection. Figure~\ref{fig:depth-ablation} evaluates \(D\in\{4,5,6,7,8\}\) and \(\eta\in\{0.905,0.910,0.915\}\) on the representative split. Appendix E.1 provides the numerical results.

\begin{figure}[t]
\centering
\includegraphics[width=\columnwidth]{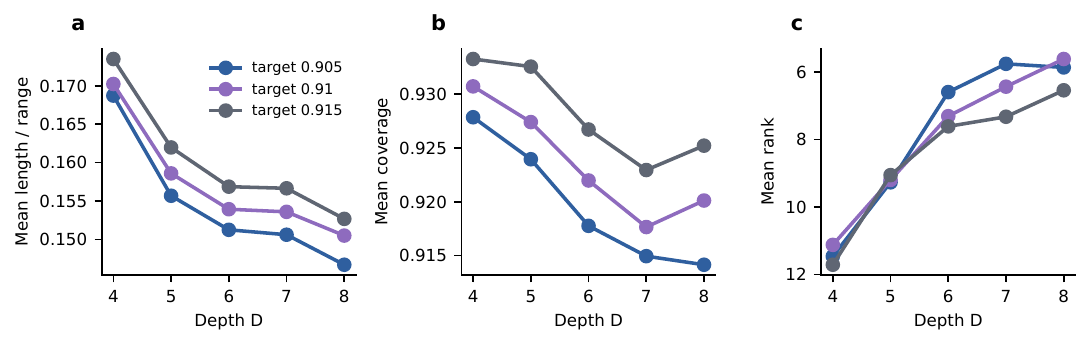}
\caption{Depth and validation-target scaling on the representative split: (a) mean normalized length, (b) mean test coverage, and (c) coverage-constrained mean rank.}
\label{fig:depth-ablation}
\end{figure}

Figure~\ref{fig:depth-ablation} shows the same depth trend at all three validation targets. At \(\eta=0.905\), increasing \(D\) from 4 to 8 reduces mean length/range from \(0.1687\) to \(0.1467\). Mean rank improves from \(11.46\) to \(5.86\), while coverage remains above \(0.914\). The trends remain consistent at \(\eta=0.910\) and \(0.915\). At \(D=8\), increasing \(\eta\) from \(0.905\) to \(0.915\) raises coverage from \(0.9141\) to \(0.9252\). The corresponding length increase is limited, from \(0.1467\) to \(0.1527\). These results establish depth as an effective scaling dimension for interval resolution and support \(D=8,\eta=0.905\) as the main compact setting.

\begin{figure*}[t]
\centering
\includegraphics[width=0.90\textwidth]{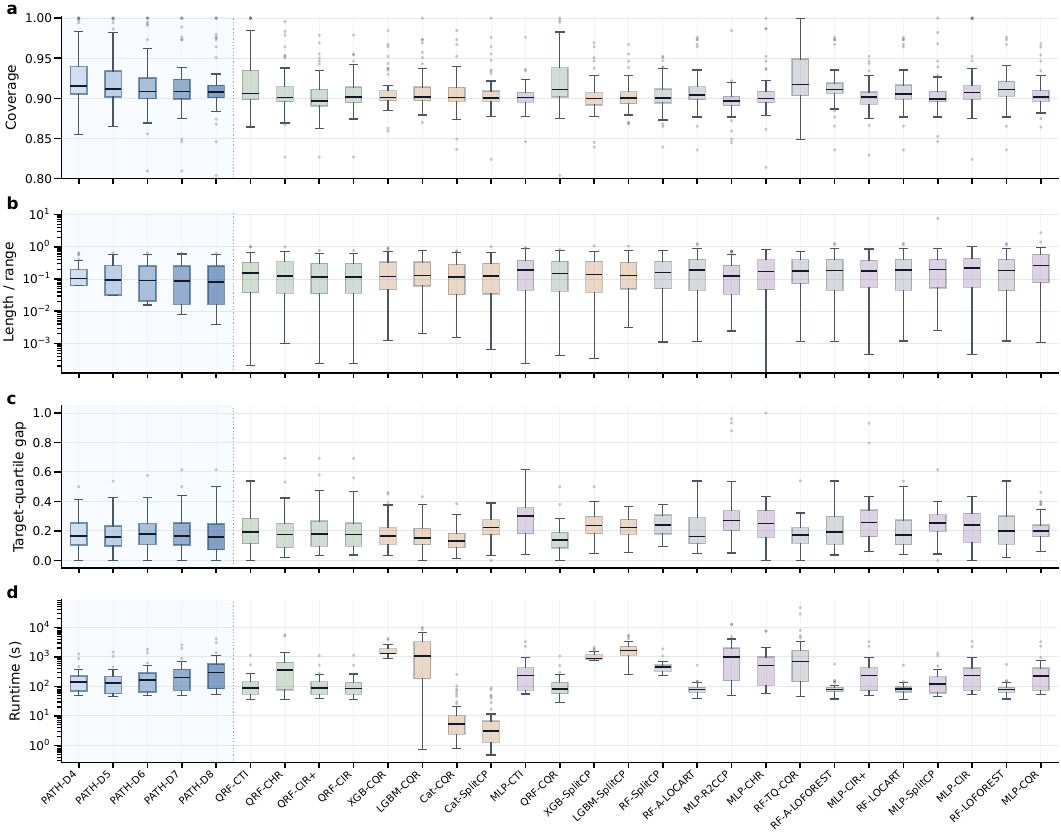}
\caption{Representative-split distributions for \method depths \(D=4,\ldots,8\) and 24 baselines: coverage, normalized length, target-quartile coverage gap, and runtime. Shading marks \method; length and runtime use logarithmic scales.}
\label{fig:full-distribution-diagnostics}
\end{figure*}

\subsection{Modeling Ablation}

Table~\ref{tab:modeling-ablation} isolates hierarchical output modeling and autoregressive refinement on \benchmark over 10 seeds. All variants use \(D=8\), \(\eta=0.905\), and identical selection and extraction. \method{} w/o AR retains the hierarchy, whereas Direct Transformer predicts leaves directly.

\begin{table}[t]
\centering
\caption{10-seed modeling ablation on \benchmark at \(D=8\), \(\eta=0.905\). Shorter (/56) counts datasets where \method has lower seed-averaged length.}
\label{tab:modeling-ablation}
\scriptsize
\setlength{\tabcolsep}{0.8pt}
\renewcommand{\arraystretch}{1.00}
\begin{tabular}{@{}L{0.32\columnwidth}C{0.11\columnwidth}C{0.16\columnwidth}C{0.16\columnwidth}C{0.19\columnwidth}@{}}
\toprule
Method & Cov. & \shortstack{Mean\\len/rng \(\downarrow\)} & \shortstack{Median\\len/rng \(\downarrow\)} & \shortstack{PATH shorter\\datasets \(\uparrow\)} \\
\midrule
\textbf{\method} & 0.9144 & \textbf{0.1473} & \textbf{0.0547} & -- \\
\method{} w/o AR & 0.9197 & 0.1506 & 0.0566 & 42/56 \\
Direct Transformer & 0.9195 & 0.1623 & 0.0625 & 52/56 \\
\bottomrule
\end{tabular}
\end{table}

All variants maintain mean coverage above \(0.914\). Removing autoregressive refinement raises mean length/range from \(0.1473\) to \(0.1506\), and direct leaf prediction raises it to \(0.1623\). Full \method is shorter on 42/56 and 52/56 datasets, with mean reductions of \(2.2\%\) and \(8.7\%\). The hierarchy provides the larger gain, while autoregression adds a consistent improvement.

\subsection{Distributional Diagnostics}

Figure~\ref{fig:full-distribution-diagnostics} compares coverage, normalized length, target-quartile gap (maximum minus minimum quartile coverage), and runtime. Greater depth shortens \method intervals while preserving near-target coverage and comparable gaps, at increased runtime.

\subsection{Further Analysis}

Appendix D gives aggregate and interval-score results. Appendix E.1 confirms depth and target scaling. Appendix E.2 measures validation-to-test transfer. Appendix E.3 favors direct extraction over scalar expansion. Appendix E.4 gives all-baseline pairwise comparisons. Appendix E.5 confirms gains under strict coverage. Appendix E.6 covers dataset factors, conditional coverage, and runtime. Appendix F lists per-dataset records.

\section{Conclusion}

Interval prediction often shapes final intervals after training from a surrogate. \method instead models their inherent hierarchical geometry. We formulate this structure as next-interval prediction: at each depth, the model predicts probability flow from a coarse interval to nested subintervals, linking broad uncertainty to boundary placement. Across \benchmark, 10 random seeds, and 24 baselines, \method attains the shortest mean normalized intervals while mean coverage remains above 0.9. Ablations confirm the benefits of finer hierarchies and autoregressive refinement. This establishes hierarchical geometry as a principle for compact intervals.

\clearpage
{\small
\bibliography{aaai2027}
}

\clearpage
\onecolumn
\appendix
\raggedbottom

\numberwithin{figure}{section}
\numberwithin{table}{section}
\numberwithin{algorithm}{section}

\section*{Appendix}
\subsection*{Overview}
This appendix supports the next-interval prediction formulation with the theoretical properties and algorithm, the complete evaluation protocol, numerical ablations, interval-score and strict-coverage results, validation-selection diagnostics, dataset-level analyses, runtime comparisons, and per-dataset records. The main comparison and model-design ablation average results on \benchmark over 10 random seeds, while depth/target sweeps and detailed diagnostics use the representative split unless stated otherwise. \textbf{The accompanying arXiv package also includes the code and reproducibility assets}, comprising the training and evaluation scripts, the OpenML dataset manifest, and the random-seed protocol.

\section{Basic Properties of the Construction}
\label{app:basic-properties}

At each tree depth, next-interval prediction models how probability mass flows from a current interval to its nested children. This section records properties used to interpret the resulting \method construction. They establish that the tree decoder defines a valid target distribution, that the non-autoregressive setting is an exact special case, and that depth controls the grid approximation of contiguous interval extraction.

\begin{proposition}[Normalized tree distribution]
\label{prop:normalized-tree-distribution}
Fix \(x\), \(T\), and \(\lambda\). If every branch probability \(\pi^\lambda_{t,a}(x)\) lies in \([0,1]\), then the refined leaf distribution defined by the tree product satisfies
\[
  p_{\lambda,T}(b\mid x)\ge 0
  \quad\text{for all } b,
  \qquad
  \sum_{b=0}^{B-1} p_{\lambda,T}(b\mid x)=1 .
\]
\end{proposition}

\begin{proof}
Assign mass \(m_0(\emptyset)=1\) to the root. For a node \(a\) at depth \(t\), assign its left and right child masses as
\[
  m_{t+1}(a0)=m_t(a)\bigl(1-\pi^\lambda_{t,a}(x)\bigr),
  \qquad
  m_{t+1}(a1)=m_t(a)\pi^\lambda_{t,a}(x).
\]
Both child masses are nonnegative, and their sum is \(m_t(a)\). By induction over levels, the total mass over all nodes at depth \(t\) is one for every \(t\). At depth \(D\), each node corresponds to one leaf, and the recursive mass assigned to that leaf is exactly the product defining \(p_{\lambda,T}(b\mid x)\). Therefore the leaf probabilities are nonnegative and sum to one.
\end{proof}

\begin{proposition}[Direct distribution as a special case]
\label{prop:direct-special-case}
Assume the base leaf distribution satisfies \(q_T(b\mid x)>0\) for all leaves. When \(\lambda=0\), the refined distribution equals the direct leaf distribution:
\[
  p_{0,T}(b\mid x)=q_T(b\mid x),
  \qquad b=0,\ldots,B-1 .
\]
\end{proposition}

\begin{proof}
For any tree node \(\mathcal{C}_t(a)\), the branch probability \(\pi^0_{t,a}(x)\) is the conditional probability under \(q_T\) that the target leaf lies in the right child, given that it lies in the parent node. The left branch probability is the corresponding conditional probability for the left child. For a leaf \(b\), multiplying these conditional probabilities along its root-to-leaf path gives
\[
  \prod_{t=0}^{D-1}
  q_T\!\left(\mathcal{C}_{t+1}(a_{t+1}(b)) \mid
  \mathcal{C}_t(a_t(b)), x\right)
  =
  q_T(\{b\}\mid x),
\]
because the intermediate parent masses telescope and the root mass is one. This product is \(p_{0,T}(b\mid x)\), so \(p_{0,T}(b\mid x)=q_T(b\mid x)\).
\end{proof}

\begin{proposition}[Optimal contiguous extraction and grid error]
\label{prop:contiguous-grid-error}
For a fixed discrete distribution \(p(b\mid x)\) and mass \(M\), the interval extraction rule in the main text returns a shortest interval among all contiguous leaf intervals whose mass is at least \(M\). Moreover, let \(P\) be any distribution on a bounded target range of length \(R\), and let \(L^*(M)\) be the shortest length of any continuous interval with \(P\)-mass at least \(M\). Let \(L_D^*(M)\) be the shortest length among intervals whose endpoints are restricted to a uniform grid with \(B=2^D\) cells and width \(\Delta=R/2^D\). Then
\[
  L^*(M)\le L_D^*(M)\le L^*(M)+2\Delta .
\]
\end{proposition}

\begin{proof}
The first statement follows directly from enumerating the feasible set of contiguous leaf intervals and minimizing their endpoint distance.

For the grid approximation, every grid-aligned interval is also a continuous interval, so \(L^*(M)\le L_D^*(M)\). Conversely, let \([u,v]\) be a continuous interval with \(P([u,v])\ge M\) and length arbitrarily close to \(L^*(M)\). Choose grid points \(e_i\) and \(e_j\) such that \(e_i\le u<e_{i+1}\) and \(e_{j-1}<v\le e_j\). Then \([e_i,e_j]\) is grid-aligned, contains \([u,v]\), and therefore has mass at least \(M\). Its length is at most \((v-u)+2\Delta\). Taking \([u,v]\) arbitrarily close to optimal gives \(L_D^*(M)\le L^*(M)+2\Delta\).
\end{proof}

\begin{proposition}[Finite-grid validation concentration]
\label{prop:validation-concentration}
Condition on a fixed trained model and a fixed finite inference grid \(\mathcal{A}\) of size \(K\). For each candidate \(a\in\mathcal{A}\), let \(c_a=\mathbb{P}\{Y\in I_a(X)\}\) be its population coverage and let \(\hat c_a\) be its empirical coverage on an independent validation set of size \(n_{\mathrm{val}}\). Then, with probability at least \(1-\delta\),
\[
  \max_{a\in\mathcal{A}} |\hat c_a-c_a|
  \le
  \sqrt{\frac{\log(2K/\delta)}{2n_{\mathrm{val}}}} .
\]
\end{proposition}

\begin{proof}
For a fixed candidate \(a\), the coverage indicators
\(\mathbf{1}\{Y_i\in I_a(X_i)\}\) are bounded Bernoulli variables with mean \(c_a\). Hoeffding's inequality gives
\[
  \mathbb{P}\{|\hat c_a-c_a|>\epsilon\}\le 2\exp(-2n_{\mathrm{val}}\epsilon^2).
\]
Applying a union bound over the \(K\) grid candidates and setting the right-hand side to \(\delta\) yields the stated bound.
\end{proof}

\section{PATH Algorithm}
\label{app:path-algorithm}

Algorithm~\ref{alg:path} summarizes next-interval prediction in the complete training and inference pipeline. Each operation points to the corresponding equation in the main text.

\begin{algorithm}[t]
\caption{Pseudocode for \method. Each operation is linked to the corresponding equation in the main text.}
\label{alg:path}
\begin{algorithmic}[1]
\State Fit uniform target boundaries \(e_b\) on the training split using Eq.~\eqref{eq:leaf-boundaries}.
\State Encode each target as a clipped leaf index \(z_i\) using Eq.~\eqref{eq:leaf-index}.
\State Map each tree prefix to its target region \(\mathcal{C}_t(a)\) using Eq.~\eqref{eq:tree-cell}.
\State Predict the base leaf distribution \(q_T(b\mid x)\) using Eq.~\eqref{eq:base-leaf-distribution}.
\State Convert base leaf mass into branch probabilities \(\pi^0_{t,a}(x)\) using Eq.~\eqref{eq:base-branch-probability}.
\State Refine branch probabilities with residual autoregression using Eq.~\eqref{eq:residual-branch-probability}.
\State Enumerate root-to-leaf paths to obtain \(p_{\lambda,T}(b\mid x)\) using Eq.~\eqref{eq:refined-leaf-distribution}.
\State Train the direct and refined distributions with Eqs.~\eqref{eq:leaf-loss}--\eqref{eq:training-objective}.
\State Extract the shortest contiguous mass interval using Eq.~\eqref{eq:shortest-mass-interval}.
\State Select \(T\), \(M\), and \(\lambda\) on validation by evaluating intervals from Eq.~\eqref{eq:shortest-mass-interval}.
\end{algorithmic}
\end{algorithm}

Algorithm~\ref{alg:path} is the procedural counterpart of the method section. It separates the construction into target encoding, distribution prediction, autoregressive refinement, training, interval extraction, and validation selection, making clear that \method returns an interval only after the distributional model and the validation-selected inference parameters have been fixed.

\section{Evaluation Protocol}
\label{app:evaluation-protocol}

\paragraph{Data splits.}
The main comparison uses random seeds \(0,\ldots,9\), each defining an independent 60/20/20 train/validation/test split. The training split fits preprocessing and target bin edges, the validation split selects checkpoints or inference parameters, and the fixed choices are then evaluated on the test split. The model-design ablation uses the same 10 random seeds; depth/target ablations and diagnostic sections use the representative split unless stated otherwise.

\paragraph{Computing environment.}
Experiments were executed on 64-bit Linux cluster nodes (kernel 5.15). GPU workloads used NVIDIA A800 or H800 accelerators with 80\,GB memory, while CPU-only baselines ran on Intel Xeon Platinum-class processors; a representative CPU node contained two Intel Xeon Platinum 8468 processors with 96 physical cores in total. The primary software environment used Python 3.10.18, PyTorch 2.6.0 with CUDA 11.8, NumPy 2.1.2, pandas 2.3.0, scikit-learn 1.7.0, SciPy 1.15.3, CatBoost 1.2.10, XGBoost 2.1.4, LightGBM 3.3.5, Matplotlib 3.10.3, and Seaborn 0.13.2. The anonymous code package provides the complete Python dependency list in \texttt{requirements.txt}.

\paragraph{OpenML dataset manifest.}
Table~\ref{tab:openml-dataset-manifest} lists the exact OpenML data identifiers and target columns used in the 56-dataset benchmark. Dataset names follow the aliases used throughout the paper, while the data IDs uniquely identify the corresponding OpenML sources.

\begingroup
\footnotesize
\setlength{\tabcolsep}{3.0pt}
\renewcommand{\arraystretch}{1.05}
\begin{longtable}{@{}L{0.36\textwidth}C{0.08\textwidth}L{0.30\textwidth}R{0.10\textwidth}C{0.06\textwidth}@{}}
\caption{OpenML dataset manifest. Rows and features refer to the full dataset before the split protocol is applied.}
\label{tab:openml-dataset-manifest}\\
\toprule
Dataset & ID & Target & Rows & Feat. \\
\midrule
\endfirsthead
\multicolumn{5}{c}{\footnotesize Table~\thetable: OpenML manifest, continued.}\\
\toprule
Dataset & ID & Target & Rows & Feat. \\
\midrule
\endhead
\midrule
\multicolumn{5}{r}{\footnotesize Continued on next page}\\
\endfoot
\bottomrule
\endlastfoot
Allstate\_Claims\_Severity & 42571 & loss & 188,318 & 130 \\
BNG(lowbwt) & 1193 & class & 31,104 & 9 \\
Bike\_Sharing\_Demand & 44000 & count & 17,379 & 6 \\
Brazilian\_houses & 42688 & total\_(BRL) & 10,692 & 12 \\
CPMP-2015-regression & 41700 & runtime & 2,108 & 26 \\
CPS1988 & 43963 & wage & 28,155 & 6 \\
Job\_Profitability & 44311 & Jobs\_Gross\_Margin\_Percentage & 14,480 & 29 \\
Meta\_Album\_MD\_5\_BIS\_Mini & 44296 & CATEGORY & 28,240 & 38 \\
Meta\_Album\_MD\_6\_Mini & 44310 & CATEGORY & 28,120 & 46 \\
Meta\_Album\_MD\_MIX\_Mini & 44287 & CATEGORY & 28,240 & 68 \\
Minneapolis-Air-Quality-Survey & 43747 & Results & 4,790 & 17 \\
Moneyball & 41021 & RS & 1,232 & 14 \\
NASA\_PHM2008 & 42821 & class & 45,918 & 21 \\
OnlineNewsPopularity & 4545 & shares & 39,644 & 60 \\
Titanic & 41265 & Fare & 1,307 & 7 \\
aids2 & 46158 & time & 2,814 & 5 \\
appliances\_energy\_prediction & 46283 & Appliances & 19,735 & 27 \\
arsenic-female-lung & 513 & events & 559 & 4 \\
arsenic-male-bladder & 482 & events & 559 & 4 \\
arsenic-male-lung & 536 & events & 559 & 4 \\
beijing\_pm & 46285 & pm2.5 & 41,757 & 11 \\
bengaluru\_real\_estate\_price & 46284 & price & 13,320 & 7 \\
black\_friday & 41540 & Purchase & 166,821 & 9 \\
cat\_adoption & 47176 & time & 2,257 & 19 \\
check\_times & 47177 & time & 13,626 & 23 \\
cholesterol & 204 & chol & 303 & 13 \\
chscase\_vine2 & 689 & col\_3 & 468 & 2 \\
climate\_change\_impact\_on\_\newline agriculture\_2024 & 46726 & Economic\_Impact\_\newline Million\_USD & 10,000 & 14 \\
colrec & 46145 & time & 5,578 & 6 \\
cpu & 561 & class & 209 & 7 \\
delays\_zurich\_transport & 42495 & delay & 27,327 & 17 \\
diamonds & 42225 & price & 53,940 & 9 \\
flchain & 46161 & time & 4,000 & 8 \\
forest\_fires & 42363 & area & 517 & 12 \\
fps\_benchmark & 44992 & FPS & 24,624 & 43 \\
grace & 46168 & time & 1,000 & 7 \\
hdfail & 46167 & time & 52,410 & 6 \\
health\_insurance & 44993 & whrswk & 22,272 & 11 \\
house\_16H & 574 & price & 22,784 & 16 \\
house\_8L & 218 & price & 22,784 & 8 \\
houses & 537 & median\_house\_value & 20,640 & 8 \\
kings\_county & 44989 & price & 21,613 & 21 \\
laser & 42364 & Output & 993 & 4 \\
meta & 566 & class & 528 & 21 \\
metabric & 46142 & time & 1,903 & 10 \\
munich-rent-index-1999 & 46772 & rent & 3,082 & 8 \\
nafld1 & 46164 & time & 12,588 & 6 \\
pol & 201 & foo & 15,000 & 48 \\
rainfall\_bangladesh & 41539 & Rainfall & 16,755 & 3 \\
sarcos & 44976 & V22 & 48,933 & 21 \\
seoul\_bike\_sharing\_demand & 46297 & 13 & 8,760 & 13 \\
seoul\_bike\_sharing\_demand\_cat & 46328 & rented\_bike\_count & 8,760 & 17 \\
socmob & 541 & counts\_for\_sons\_\newline current\_occupation & 1,156 & 5 \\
stock\_fardamento02 & 42545 & qts & 6,277 & 6 \\
video\_transcoding & 44974 & utime & 68,784 & 18 \\
visualizing\_galaxy & 690 & velocity & 323 & 4 \\
\end{longtable}
\endgroup

The manifest contains 56 unique OpenML data IDs and spans 1,216,045 observations. Dataset sizes range from 209 to 188,318 rows and from 2 to 130 input features. Reporting the identifiers and target columns fixes the benchmark source independently of local dataset naming.

\paragraph{Main \method setting.}
The main setting is \(D=8\) with validation selection target \(\eta=0.905\). The checkpoint is selected by validation loss. The inference grid is the mass, temperature, and autoregressive residual strength grid stated in the main text. Validation selects \(T\), \(M\), and \(\lambda\), and the resulting shortest interval satisfying the selected mass is evaluated directly on the test split.

\paragraph{Ranking rule.}
For each dataset, methods are ranked by the following deterministic rule:
\begin{enumerate}
  \item Methods with empirical test coverage at least \(0.885\) are valid under the coverage rule.
  \item Valid methods are ranked by average interval length normalized by the training target range, with shorter intervals preferred.
  \item Methods with coverage below \(0.885\) are ranked after all valid methods by coverage error to the valid threshold and then length.
  \item Wins count first-place finishes under this rule for the relevant evaluation.
\end{enumerate}

\paragraph{Reporting boundary.}
The depth and validation-target ablation characterizes target resolution and coverage operating points. The main text uses \(D=8,\eta=0.905\) as the compact high-resolution setting and presents larger validation targets as more conservative alternatives. Each experiment fixes depth and validation target before test evaluation and applies the same setting to every dataset.

\section{Aggregate Diagnostics}
\label{app:aggregate-diagnostics}

This section reports the aggregate metrics omitted from the main tables. All values average \benchmark over 10 random seeds unless stated otherwise. The representative-split distributional boxplots appear in Figure~\ref{fig:full-distribution-diagnostics} of the main paper.

\begin{figure}
\centering
\includegraphics[width=\linewidth]{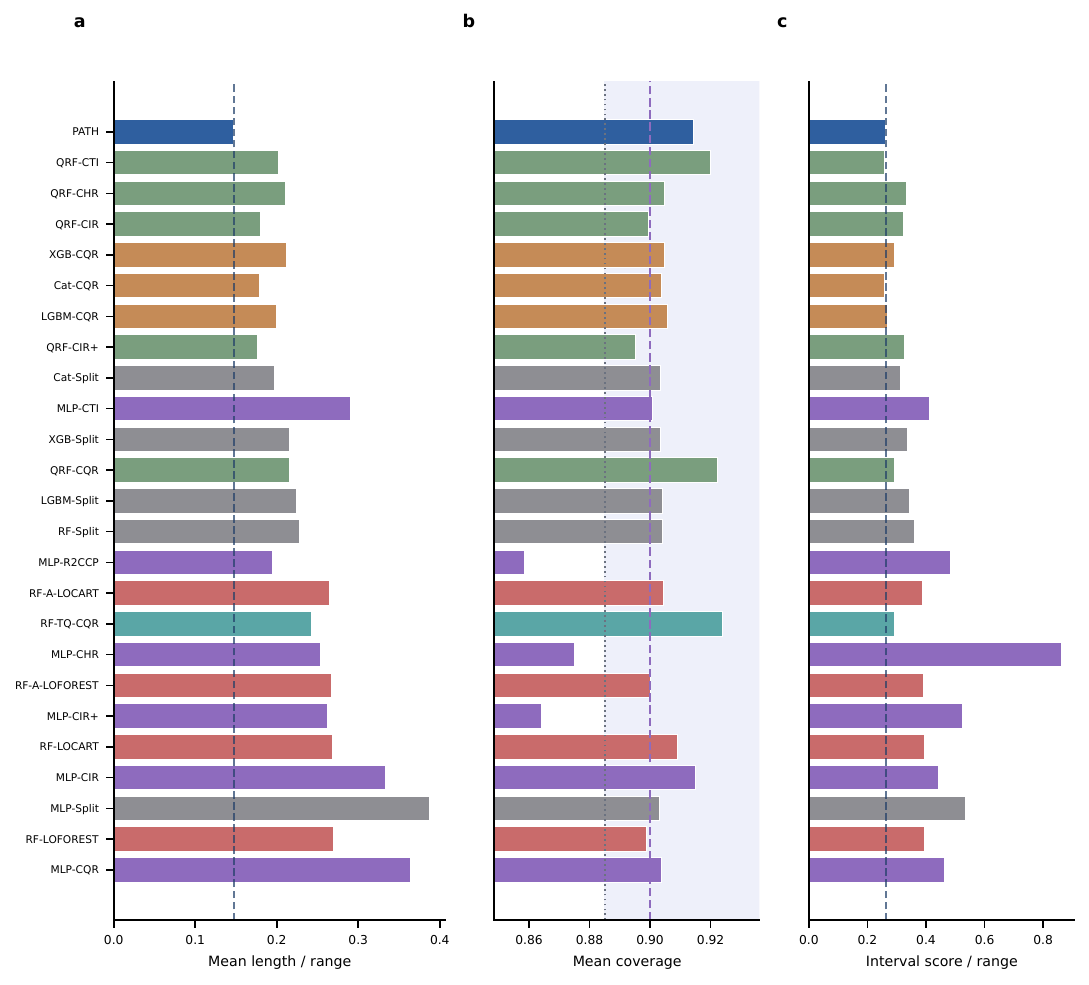}
\caption{Aggregate main metric summary over all 25 concrete methods and 10 random seeds, ordered as in Table~\ref{tab:cross-dataset-summary}. (a) Mean interval length normalized by the training target range. (b) Mean empirical coverage, with the nominal 0.9 target and the coverage threshold marked. (c) Mean interval score normalized by the training target range. Lower is better in panels (a) and (c).}
\label{fig:main-results}
\end{figure}

Figure~\ref{fig:main-results} compares length, coverage, and interval score for all 25 methods under the 10-split protocol. \method attains the shortest mean normalized length, \(0.1473\), while maintaining mean coverage \(0.9144\). The next shortest method has mean length \(0.1760\), so \method reduces normalized length by \(16.3\%\) relative to this closest length competitor. Several baselines achieve greater mean coverage, but only by returning substantially longer intervals. The three panels therefore show that \method occupies the most compact region without falling below the nominal mean coverage target.

\begin{table}
\centering
\caption{Supplementary interval-score summary for the 10-split main comparison. Score/rng is the mean interval score normalized by the training target range. Rank and wins follow the coverage-constrained length rule in Table~\ref{tab:cross-dataset-summary}.}
\label{tab:interval-score-summary-current}
\footnotesize
\setlength{\tabcolsep}{3.4pt}
\begin{tabular*}{0.92\linewidth}{@{\extracolsep{\fill}}lrrrrr@{}}
\toprule
Method & mean cov & mean len/rng \(\downarrow\) & score/rng \(\downarrow\) & rank \(\downarrow\) & wins \(\uparrow\) \\
\midrule
\textbf{\method} & 0.9144 & 0.1473 & 0.2627 & 5.25 & 163 \\
QRF-CTI & 0.9201 & 0.2017 & 0.2578 & 8.38 & 130 \\
QRF-CHR & 0.9046 & 0.2110 & 0.3339 & 9.03 & 12 \\
QRF-CIR & 0.8995 & 0.1800 & 0.3243 & 9.45 & 9 \\
XGBoost-CQR & 0.9047 & 0.2120 & 0.2936 & 9.96 & 6 \\
CatBoost-CQR & 0.9038 & 0.1786 & 0.2571 & 10.03 & 9 \\
LightGBM-CQR & 0.9059 & 0.2002 & 0.2698 & 10.29 & 7 \\
QRF-CIR+ & 0.8951 & 0.1760 & 0.3264 & 10.73 & 32 \\
CatBoost-SplitCP & 0.9034 & 0.1967 & 0.3116 & 11.94 & 26 \\
MLP-CTI & 0.9009 & 0.2900 & 0.4135 & 12.67 & 49 \\
XGBoost-SplitCP & 0.9035 & 0.2162 & 0.3363 & 12.77 & 27 \\
QRF-CQR & 0.9222 & 0.2159 & 0.2914 & 12.99 & 0 \\
LightGBM-SplitCP & 0.9042 & 0.2242 & 0.3452 & 13.91 & 14 \\
RF-SplitCP & 0.9043 & 0.2277 & 0.3605 & 14.08 & 2 \\
MLP-R2CCP & 0.8586 & 0.1953 & 0.4833 & 14.51 & 40 \\
RF-A-LOCART & 0.9043 & 0.2649 & 0.3894 & 14.65 & 4 \\
RF-TQ-CQR & 0.9241 & 0.2431 & 0.2911 & 14.83 & 2 \\
MLP-CHR & 0.8750 & 0.2533 & 0.8616 & 14.92 & 12 \\
RF-A-LOFOREST & 0.9001 & 0.2668 & 0.3901 & 15.32 & 0 \\
MLP-CIR+ & 0.8643 & 0.2626 & 0.5237 & 15.48 & 12 \\
RF-LOCART & 0.9090 & 0.2682 & 0.3944 & 15.78 & 0 \\
MLP-CIR & 0.9149 & 0.3330 & 0.4430 & 15.90 & 1 \\
MLP-SplitCP & 0.9031 & 0.3869 & 0.5340 & 16.24 & 2 \\
RF-LOFOREST & 0.8988 & 0.2700 & 0.3962 & 16.80 & 0 \\
MLP-CQR & 0.9037 & 0.3646 & 0.4644 & 19.09 & 1 \\
\bottomrule
\end{tabular*}
\end{table}

Table~\ref{tab:interval-score-summary-current} places interval score beside coverage-constrained interval efficiency. CatBoost-CQR and QRF-CTI obtain the two lowest score/rng values, \(0.2571\) and \(0.2578\). The \method score, \(0.2627\), is within \(2.2\%\) of the best value, while its intervals are \(17.5\%\) shorter than CatBoost-CQR and \(27.0\%\) shorter than QRF-CTI in aggregate. \method also retains the best rank and largest win count. Thus, its large compactness gain is achieved while remaining close to the strongest interval-score methods.

\section{Additional Diagnostics}
\label{app:additional-mined-diagnostics}

This section expands the main experiments through numerical depth and validation-target results, validation-selection diagnostics, scalar-margin compatibility, pairwise robustness tests, conditional coverage analysis, and runtime comparisons. Unless explicitly identified as a 10-split aggregate, results use the representative split.

\subsection{Depth and Validation-Target Results}
\label{app:depth-target-details}

\begin{table}
\centering
\caption{Representative-split depth and validation-target ablation corresponding to Figure~\ref{fig:depth-ablation}. Rank follows the coverage-constrained length rule; lower is better.}
\label{tab:depth-target-ablation}
\footnotesize
\setlength{\tabcolsep}{3.4pt}
\begin{tabular*}{0.94\linewidth}{@{\extracolsep{\fill}}rrrrrrr@{}}
\toprule
target & \(D\) & mean cov & valid & len/rng \(\downarrow\) & rank \(\downarrow\) & score/rng \(\downarrow\) \\
\midrule
0.905 & 4 & 0.9278 & 52/56 & 0.1687 & 11.46 & 0.3047 \\
0.905 & 5 & 0.9239 & 52/56 & 0.1557 & 9.27 & 0.2799 \\
0.905 & 6 & 0.9178 & 53/56 & 0.1512 & 6.59 & 0.2754 \\
0.905 & 7 & 0.9149 & 52/56 & 0.1506 & 5.75 & 0.2843 \\
0.905 & 8 & 0.9141 & 51/56 & \textbf{0.1467} & 5.86 & 0.2647 \\
0.910 & 4 & 0.9307 & 54/56 & 0.1702 & 11.13 & 0.3021 \\
0.910 & 5 & 0.9274 & 53/56 & 0.1586 & 9.20 & 0.2781 \\
0.910 & 6 & 0.9220 & 53/56 & 0.1539 & 7.30 & 0.2762 \\
0.910 & 7 & 0.9176 & 52/56 & 0.1536 & 6.43 & 0.2806 \\
0.910 & 8 & 0.9201 & 53/56 & 0.1505 & \textbf{5.61} & 0.2626 \\
0.915 & 4 & 0.9332 & 54/56 & 0.1735 & 11.71 & 0.2973 \\
0.915 & 5 & 0.9325 & 56/56 & 0.1620 & 9.05 & 0.2759 \\
0.915 & 6 & 0.9267 & 53/56 & 0.1569 & 7.61 & 0.2741 \\
0.915 & 7 & 0.9229 & 52/56 & 0.1567 & 7.32 & 0.2783 \\
0.915 & 8 & 0.9252 & 53/56 & 0.1527 & 6.54 & 0.2593 \\
\bottomrule
\end{tabular*}
\end{table}

Table~\ref{tab:depth-target-ablation} provides the numerical values underlying Figure~\ref{fig:depth-ablation}. Increasing depth from \(D=4\) to \(D=8\) reduces mean normalized length by \(11.6\%\)--\(13.0\%\) across the three validation targets. Mean rank improves at the same time, while every setting retains mean coverage above \(0.914\). The shortest configuration is \(D=8,\eta=0.905\), whereas \(D=8,\eta=0.910\) gives the best mean rank with only a \(2.6\%\) length increase. The consistent trend establishes that finer interval hierarchies improve output resolution rather than benefiting from a single target choice.

\begin{table}
\centering
\caption{Representative-split tradeoff for the main \(D=8\) setting. Increasing the validation selection target gives a more conservative operating point and slightly lengthens intervals.}
\label{tab:setting-tradeoff-current}
\footnotesize
\setlength{\tabcolsep}{4.0pt}
\begin{tabular*}{0.94\linewidth}{@{\extracolsep{\fill}}rrrrrrrr@{}}
\toprule
target & mean cov & median cov & valid & strict & len/rng \(\downarrow\) & score/rng \(\downarrow\) & rank \(\downarrow\) \\
\midrule
0.905 & 0.9141 & 0.9083 & 51/56 & 44/56 & \textbf{0.1467} & 0.2647 & 5.86 \\
0.910 & 0.9201 & 0.9139 & 53/56 & 48/56 & 0.1505 & 0.2626 & \textbf{5.61} \\
0.915 & 0.9252 & 0.9171 & 53/56 & 50/56 & 0.1527 & 0.2593 & 6.54 \\
\bottomrule
\end{tabular*}
\end{table}

Table~\ref{tab:setting-tradeoff-current} isolates validation-target selection at \(D=8\). Raising \(\eta\) from \(0.905\) to \(0.915\) increases mean coverage from \(0.9141\) to \(0.9252\) and strict-valid datasets from 44/56 to 50/56. Mean length/range increases by only \(4.1\%\), from \(0.1467\) to \(0.1527\). The main \(\eta=0.905\) setting therefore provides the most compact operating point, while the higher targets offer predictable coverage increases at modest length cost.

\subsection{Validation Selection Diagnostics}
\label{app:validation-selection}

\begin{figure}
\centering
\includegraphics[width=\linewidth]{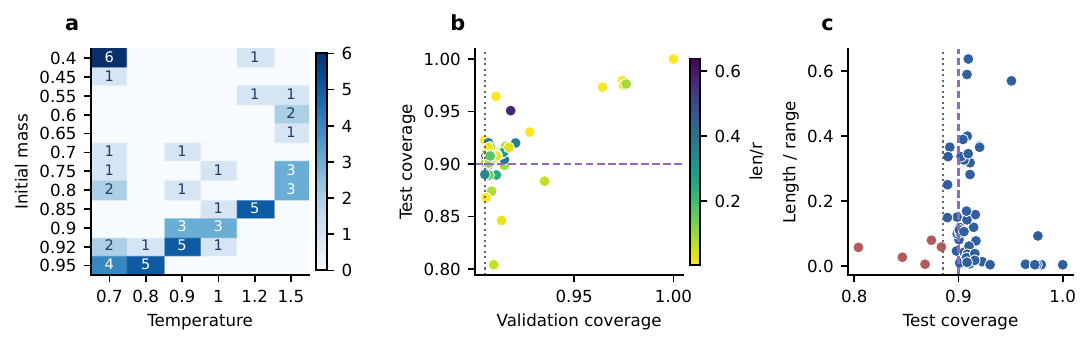}
\caption{Representative-split validation selection diagnostics for the main \(D=8,\eta=0.905\) setting. (a) Counts of initial mass and temperature pairs selected on validation data across datasets. (b) Validation coverage versus test coverage, colored by normalized test length. (c) Test coverage versus normalized interval length; blue points satisfy the coverage threshold.}
\label{fig:calibration-data-analysis}
\end{figure}

Figure~\ref{fig:calibration-data-analysis} examines how validation selection adapts the interval extractor across datasets. Panel (a) shows that selected mass and temperature values span the candidate grid rather than collapsing to one global setting. Panel (b) shows a clear validation-to-test relationship, although several datasets exhibit transfer gaps. Panel (c) places these gaps against interval length; 51/56 datasets satisfy the paper's coverage threshold at the selected \(D=8,\eta=0.905\) operating point. The result shows that dataset-specific selection is actively used and generally transfers compact intervals to the test split.

\subsection{Calibration-expanded PATH}
\label{app:calibration-expanded-path}

\method can also be combined with a final scalar expansion after interval construction. This supplementary experiment uses the same model architecture, depth \(D=8\), training schedule, random seeds, and inference grid over \(M\), \(T\), and \(\lambda\) as the main \method run. Its validation-selection rule applies scalar expansion to every candidate. For each validation candidate, \method first forms the raw shortest-mass interval \(I_0(x)=[L_0(x),U_0(x)]\) by Eq.~\eqref{eq:shortest-mass-interval}. It then computes the validation scores
\[
  s_i=\max\{L_0(x_i)-y_i,\; y_i-U_0(x_i),\;0\},
\]
sets \(q\) to the higher empirical quantile at level \(\lceil(n_{\mathrm{val}}+1)0.90\rceil/n_{\mathrm{val}}\), and evaluates the expanded interval
\[
  I_q(x)=[L_0(x)-q,\;U_0(x)+q].
\]
Candidates whose validation coverage is within \(0.015\) of the \(0.90\) target are preferred, and among them we choose the shortest validation interval; if no candidate is within tolerance, we choose the candidate with the smallest coverage error and use length to break ties. The selected \(M\), \(T\), \(\lambda\), and \(q\) are then applied to the test split. This gives a two-stage interval construction: \method determines the data-dependent interval shape through the learned target distribution, and scalar expansion adds a validation-estimated margin around that interval. The experiment evaluates this scalar margin on top of the same learned distribution, architecture, and training configuration.

\begin{table}
\centering
\caption{Supplementary comparison between the main \method setting and a calibration-expanded variant on \benchmark over 10 random seeds. Valid counts require coverage at least \(0.885\); strict counts require coverage at least \(0.900\).}
\label{tab:calibration-expanded-path}
\footnotesize
\setlength{\tabcolsep}{4.0pt}
\begin{tabular*}{0.96\linewidth}{@{\extracolsep{\fill}}lrrrrrr@{}}
\toprule
Method & mean cov & valid & strict & mean len/rng \(\downarrow\) & med len/rng \(\downarrow\) & score/rng \(\downarrow\) \\
\midrule
\method & 0.9144 & 515/560 & 429/560 & 0.1473 & \textbf{0.0547} & 0.2627 \\
Calibration-expanded \method & 0.9075 & 493/560 & 282/560 & \textbf{0.1404} & 0.0585 & 0.2825 \\
\bottomrule
\end{tabular*}
\end{table}

Table~\ref{tab:calibration-expanded-path} tests a scalar margin on top of the learned interval shape. Joint selection lowers mean length/range from \(0.1473\) to \(0.1404\), but mean coverage falls from \(0.9144\) to \(0.9075\). The direct \method interval also yields 22 more valid evaluations, 147 more strict-valid evaluations, and a lower score/rng (\(0.2627\) versus \(0.2825\)). These results favor direct distribution-based extraction as the main design: the learned interval shape provides substantially stronger coverage transfer and overall interval quality than joint scalar-margin selection.

Table~\ref{tab:original-depth-sweep-reference-current} reports an additional representative-split depth sweep at validation target \(0.900\) using the initial selection grid. This experiment tests whether the depth trend depends on the main \(0.905\) target. Its \(D=8\) setting achieves mean length/range \(0.1451\) with 49/56 valid datasets, while the main \(D=8,\eta=0.905\) setting obtains \(0.1467\) with 51/56 valid datasets.

\begin{table}
\centering
\caption{Additional representative-split depth sweep at validation target \(0.900\) using the initial selection grid.}
\label{tab:original-depth-sweep-reference-current}
\footnotesize
\setlength{\tabcolsep}{4.0pt}
\begin{tabular*}{0.92\linewidth}{@{\extracolsep{\fill}}rrrrrrr@{}}
\toprule
\(D\) & mean cov & median cov & valid & strict & len/rng \(\downarrow\) & score/rng \(\downarrow\) \\
\midrule
4 & 0.9341 & 0.9143 & 54/56 & 45/56 & 0.1708 & 0.2992 \\
5 & 0.9245 & 0.9124 & 51/56 & 40/56 & 0.1551 & 0.2770 \\
6 & 0.9219 & 0.9100 & 51/56 & 43/56 & 0.1520 & 0.2781 \\
7 & 0.9205 & 0.9089 & 51/56 & 39/56 & 0.1502 & 0.2819 \\
8 & 0.9161 & 0.9059 & 49/56 & 39/56 & 0.1451 & 0.2636 \\
\bottomrule
\end{tabular*}
\end{table}

Table~\ref{tab:original-depth-sweep-reference-current} reproduces the depth pattern at validation target \(0.900\). Increasing \(D\) from 4 to 8 reduces mean length/range by \(15.0\%\), from \(0.1708\) to \(0.1451\), while mean coverage remains \(0.9161\) at \(D=8\). Relative to this target, the main \(0.905\) setting recovers two additional valid datasets for only a \(1.1\%\) length increase. The result confirms stable depth scaling and supports the selected validation target.

\subsection{Pairwise Baseline Comparisons}
\label{app:pairwise-baseline-comparisons}

We compare \method directly with every individual baseline on the representative split to measure dataset-level rank, length, and interval-score differences. These analyses use the same frozen test results and coverage-based ranking rule as the main table.

\begin{table}
\centering
\caption{Pairwise comparison between \method and individual baselines on the representative split of \benchmark. Rank and score counts are reported out of 56 datasets, while shorter counts are reported over the datasets where both methods satisfy the coverage threshold. Length ratio is the mean \method interval length divided by the mean baseline interval length on the same both-valid subset.}
\label{tab:pairwise-dominance-current}
\footnotesize
\setlength{\tabcolsep}{3.0pt}
\begin{tabular*}{0.98\linewidth}{@{\extracolsep{\fill}}lrrrrrr@{}}
\toprule
Baseline & rank better \(\uparrow\) & rank worse \(\downarrow\) & both valid & shorter \(\uparrow\) & len ratio \(\downarrow\) & score better \(\uparrow\) \\
\midrule
MLP-SplitCP & 50/56 & 6/56 & 49/56 & 47/49 & 0.36 & 55/56 \\
MLP-R2CCP & 50/56 & 6/56 & 44/56 & 40/44 & 0.73 & 54/56 \\
MLP-CQR & 49/56 & 7/56 & 50/56 & 48/50 & 0.39 & 44/56 \\
RF-TQ-CQR & 49/56 & 7/56 & 50/56 & 48/50 & 0.60 & 28/56 \\
RF-SplitCP & 48/56 & 8/56 & 49/56 & 45/49 & 0.66 & 54/56 \\
RF-A-LOCART & 47/56 & 9/56 & 50/56 & 45/50 & 0.57 & 50/56 \\
RF-LOCART & 47/56 & 9/56 & 50/56 & 45/50 & 0.56 & 51/56 \\
RF-A-LOFOREST & 47/56 & 9/56 & 49/56 & 44/49 & 0.56 & 51/56 \\
RF-LOFOREST & 47/56 & 9/56 & 48/56 & 43/48 & 0.56 & 51/56 \\
MLP-CIR & 47/56 & 9/56 & 46/56 & 42/46 & 0.47 & 45/56 \\
MLP-CIR+ & 47/56 & 9/56 & 42/56 & 38/42 & 0.60 & 46/56 \\
LightGBM-CQR & 46/56 & 10/56 & 50/56 & 45/50 & 0.72 & 27/56 \\
QRF-CQR & 46/56 & 10/56 & 51/56 & 45/51 & 0.69 & 33/56 \\
MLP-CHR & 46/56 & 10/56 & 44/56 & 39/44 & 0.55 & 46/56 \\
CatBoost-CQR & 45/56 & 11/56 & 50/56 & 43/50 & 0.80 & 33/56 \\
XGBoost-CQR & 44/56 & 12/56 & 49/56 & 42/49 & 0.67 & 33/56 \\
XGBoost-SplitCP & 44/56 & 12/56 & 48/56 & 41/48 & 0.64 & 54/56 \\
LightGBM-SplitCP & 43/56 & 13/56 & 48/56 & 39/48 & 0.68 & 54/56 \\
CatBoost-SplitCP & 42/56 & 14/56 & 48/56 & 39/48 & 0.74 & 51/56 \\
QRF-CHR & 40/56 & 16/56 & 50/56 & 38/50 & 0.67 & 39/56 \\
QRF-CIR & 39/56 & 17/56 & 48/56 & 35/48 & 0.79 & 40/56 \\
QRF-CIR+ & 39/56 & 17/56 & 47/56 & 34/47 & 0.80 & 40/56 \\
MLP-CTI & 38/56 & 18/56 & 50/56 & 35/50 & 0.50 & 44/56 \\
QRF-CTI & 32/56 & 24/56 & 50/56 & 31/50 & 0.69 & 22/56 \\
\bottomrule
\end{tabular*}
\end{table}

Table~\ref{tab:pairwise-dominance-current} gives the numerical pairwise view behind the aggregate ranking. Across all 24 baselines, the mean valid-pair length ratio ranges from \(0.36\) to \(0.80\), and \method is shorter on a majority of jointly valid datasets in every comparison. It also ranks better on 32--50 of 56 datasets. QRF-CTI remains the closest competitor, yet \method is still shorter on 31/50 jointly valid datasets with mean length ratio \(0.69\). The pairwise results show that the aggregate advantage is not tied to one baseline family.

\begin{figure}
\centering
\includegraphics[width=\linewidth]{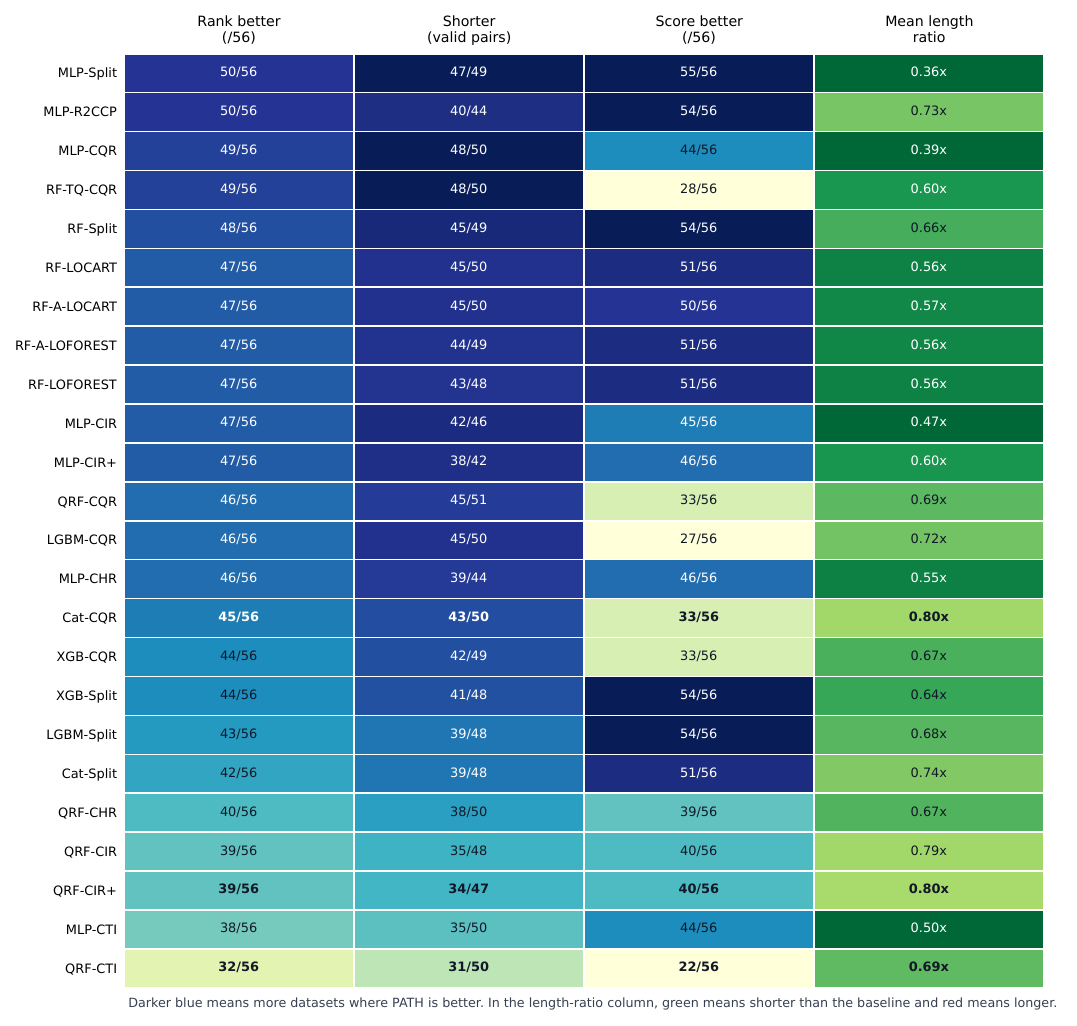}
\caption{Pairwise comparison heatmap between \method and each individual baseline on the representative split. Rank better counts datasets where \method has a better rank under the coverage rule; shorter counts datasets where both methods satisfy the coverage threshold and \method has shorter intervals; score better counts datasets with lower interval score; mean length ratio is the mean \method length divided by the mean baseline length on valid pairs, where values below one indicate shorter \method intervals.}
\label{fig:pairwise-dominance-current}
\end{figure}

Figure~\ref{fig:pairwise-dominance-current} visualizes the same comparisons across rank, valid-pair length, interval score, and mean length ratio. The broad concentration of favorable rank and length cells shows that \method's advantage extends across boosted trees, QRF variants, local random-forest methods, and neural baselines. Interval-score comparisons are more mixed against QRF-CTI, consistent with the aggregate score table, but the length ratio remains favorable for every baseline. The figure therefore localizes the main gain to robust coverage-constrained compactness.

\begin{figure}
\centering
\includegraphics[width=\linewidth]{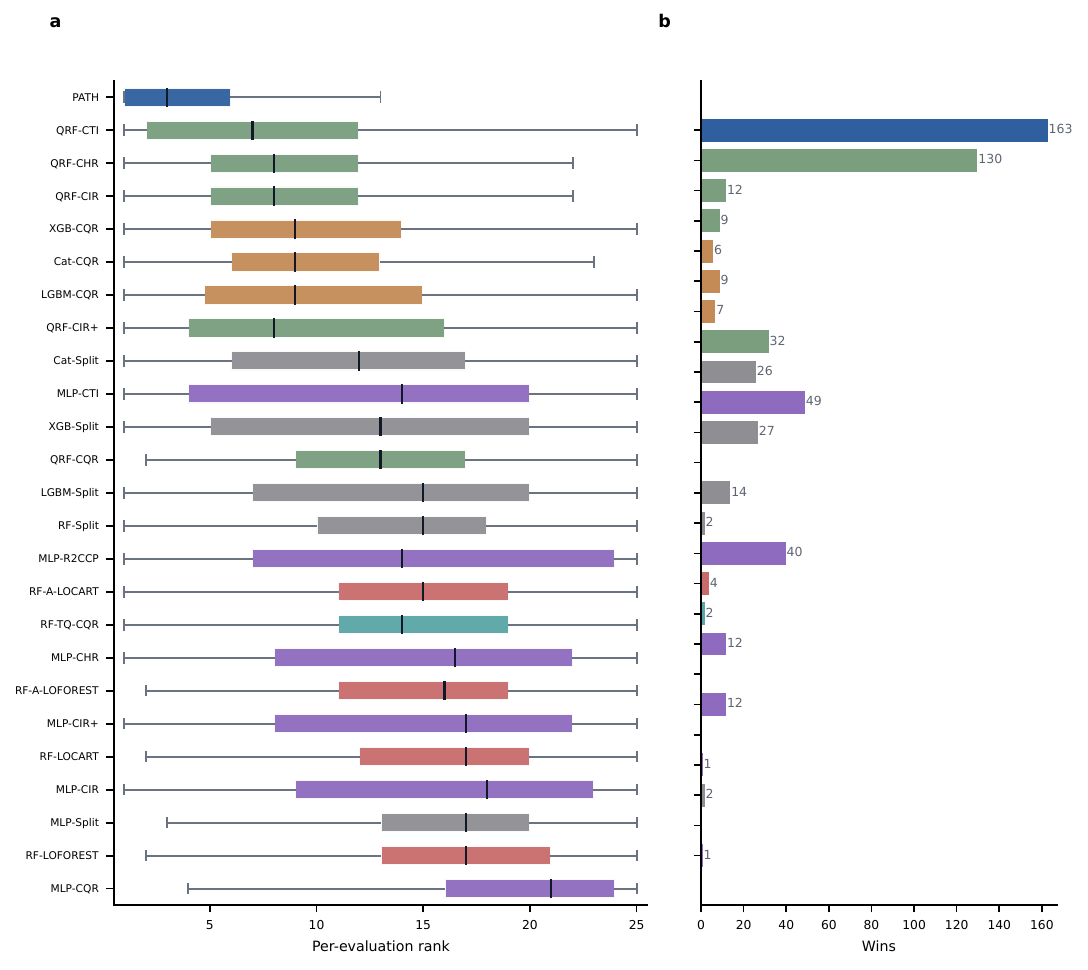}
\caption{Rank distribution for all 25 compared methods over 10 random seeds. (a) Boxplots of ranks across the 560 evaluations; lower is better. (b) Wins for the same methods.}
\label{fig:rank-profile-current}
\end{figure}

Figure~\ref{fig:rank-profile-current} shows the complete evaluation-level rank distribution rather than only its mean. The \method distribution is concentrated toward lower ranks and accompanies the largest win count, 163. Its mean rank of \(5.25\) improves on QRF-CTI's \(8.38\), while the remaining methods rank lower still. The distribution confirms that \method's aggregate lead reflects repeated high placement across datasets and splits.

\begin{figure}
\centering
\includegraphics[width=\linewidth]{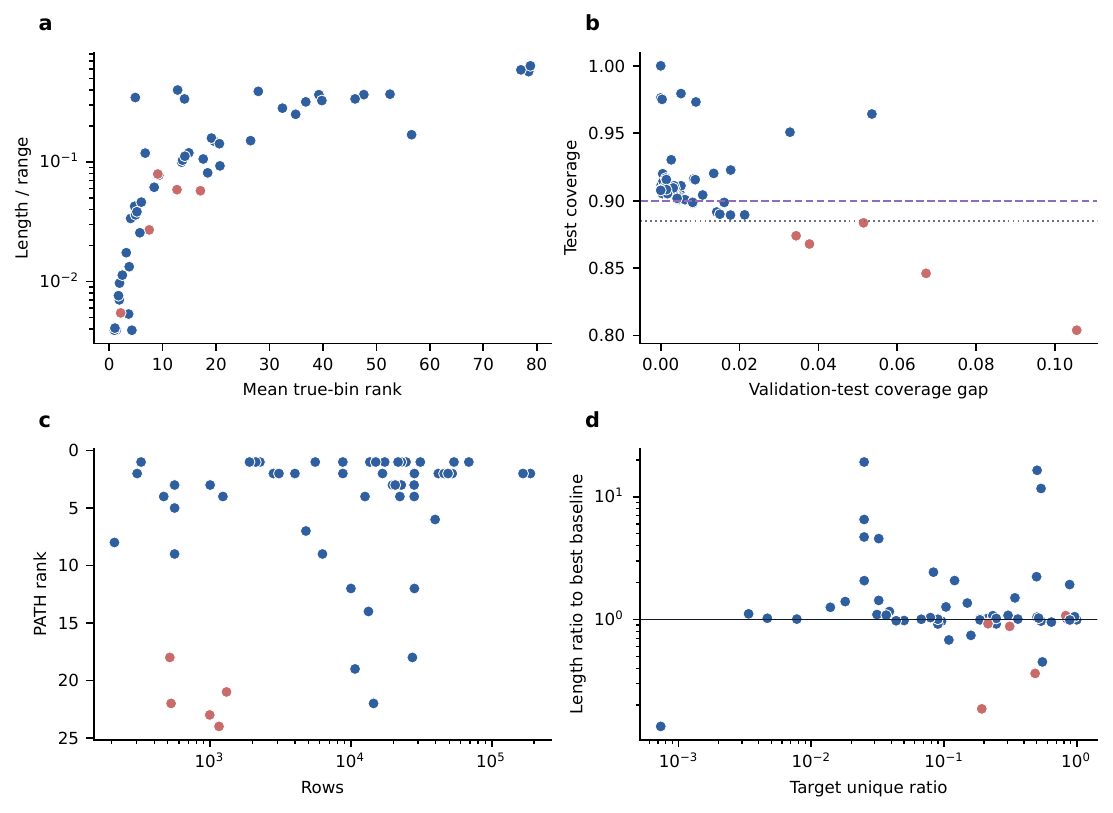}
\caption{Representative-split dataset factor diagnostics for the main \method setting. Colors distinguish datasets by the paper's coverage threshold. (a) Direct leaf ranking quality versus interval length. (b) Gap between validation and test coverage versus test coverage. (c) Sample size versus rank on each dataset. (d) Ratio of unique target values versus length ratio to the best baseline.}
\label{fig:dataset-factor-diagnostics-current}
\end{figure}

Figure~\ref{fig:dataset-factor-diagnostics-current} relates performance to leaf ordering quality, validation-to-test coverage transfer, sample size, and target discreteness. Shorter intervals and stronger ranks occur more often when the learned distribution orders nearby output regions accurately. Coverage failures align more directly with validation-to-test shifts than with sample size or target uniqueness alone. These relationships support the central mechanism: useful ordering in the output space enables compact intervals, while transfer quality governs whether the selected operating point remains valid.

\begin{figure}
\centering
\includegraphics[width=\linewidth]{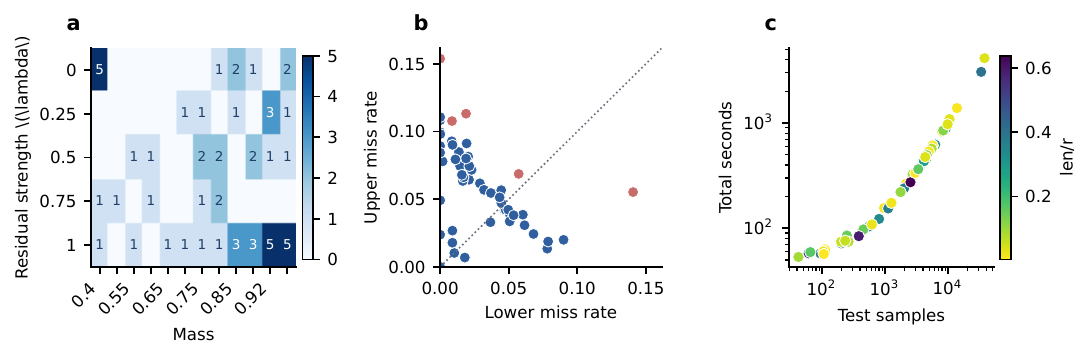}
\caption{Additional representative-split diagnostics for the main \method setting. (a) Mass selected on validation data and autoregressive residual strength. (b) Miss rates on the lower versus upper side, colored by coverage group. (c) Runtime scaling with test set size, colored by normalized interval length.}
\label{fig:path-diagnostics-current}
\end{figure}

Figure~\ref{fig:path-diagnostics-current} examines the internal choices made by \method after training. Selected mass and autoregressive residual strength vary across datasets, confirming that validation selection uses both the extraction and refinement controls rather than collapsing to one setting. The tail-miss panel isolates the asymmetric errors identified in Table~\ref{tab:coverage-diagnostics-current}, while the runtime panel shows predictable scaling with test-set size. The diagnostics link adaptive interval construction to both the observed compactness and the small number of transfer failures.

\begin{figure}
\centering
\includegraphics[width=\linewidth]{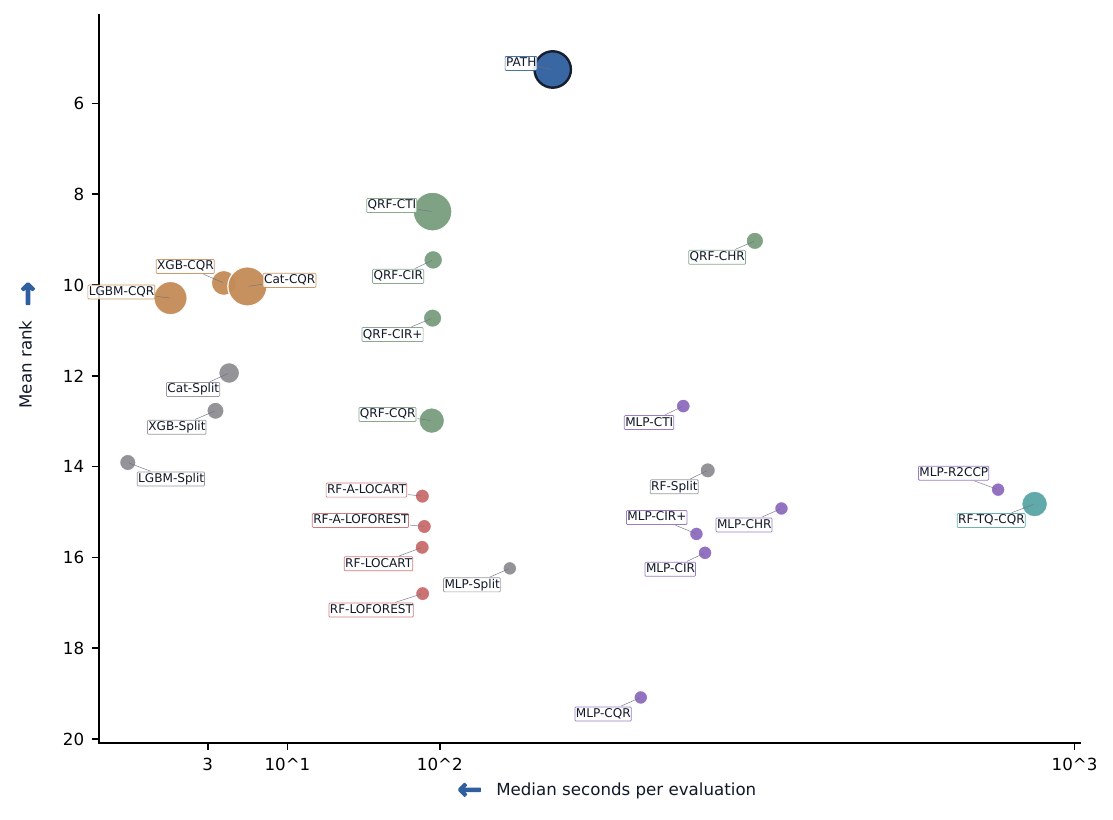}
\caption{Runtime and rank tradeoff across all 25 compared methods over 10 random seeds. Each circle is one labeled method; larger circles indicate lower aggregate score/rng. The compressed logarithmic horizontal coordinate reports median runtime per evaluation, and the vertical coordinate reports mean rank under the coverage-based ranking rule; lower rank is better.}
\label{fig:runtime-tradeoff-current}
\end{figure}

Figure~\ref{fig:runtime-tradeoff-current} compares median runtime with mean rank over the complete 10-split evaluation. \method achieves the best aggregate rank despite using more computation than the fastest tree baselines. It remains faster than several boosted-tree, neural, and TreeQuantile alternatives that obtain worse ranks. The plot therefore places \method at the strongest rank operating point among methods with comparable or greater computational cost.

\begin{figure}
\centering
\includegraphics[width=\linewidth]{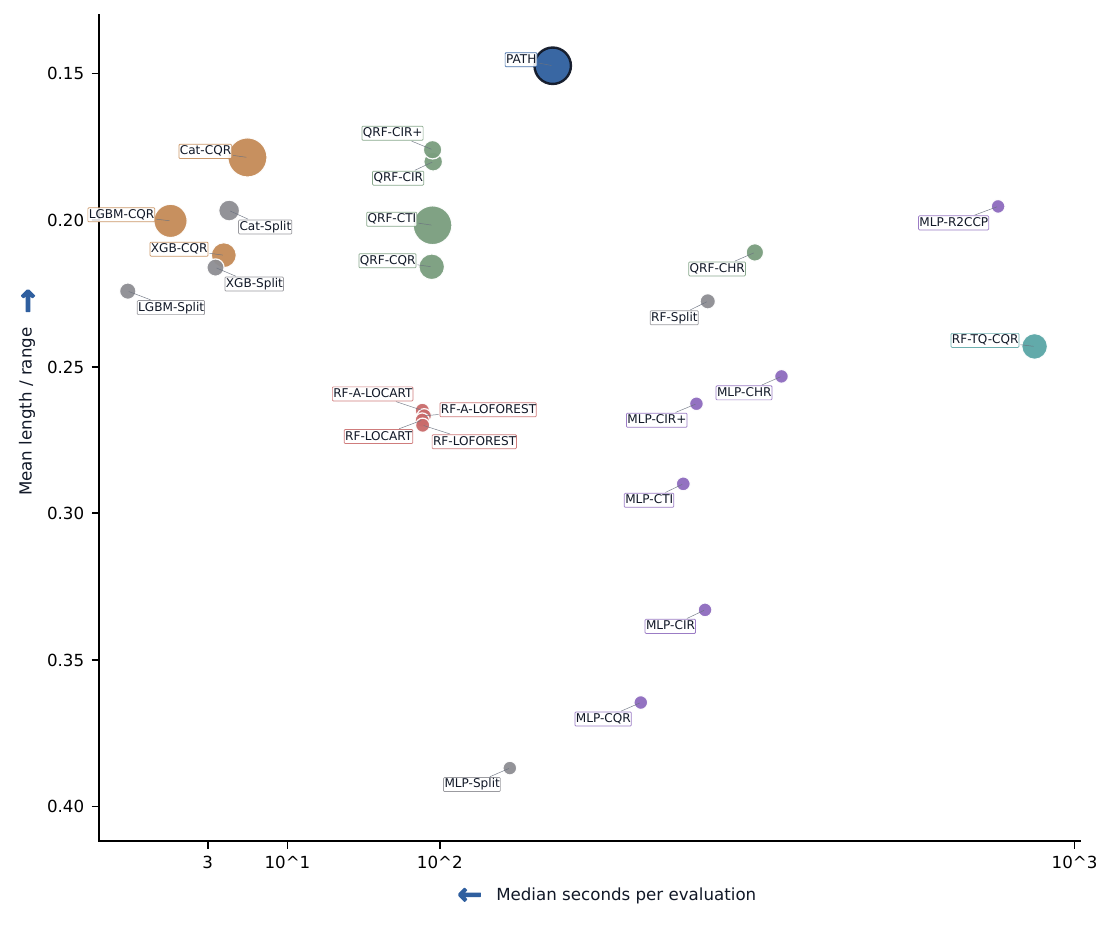}
\caption{Runtime and length tradeoff across all 25 compared methods over 10 random seeds. Each circle is one labeled method; larger circles indicate lower aggregate score/rng. The compressed logarithmic horizontal coordinate reports median runtime per evaluation, and the inverted vertical coordinate reports mean normalized interval length, so shorter intervals appear higher.}
\label{fig:runtime-length-tradeoff-current}
\end{figure}

Figure~\ref{fig:runtime-length-tradeoff-current} expresses the same comparison in interval-length terms. \method occupies the shortest-interval position across all 25 methods. Faster tree baselines return longer intervals, while several slower neural and boosted-tree alternatives also fail to match its compactness. The interval-length advantage therefore remains visible after computational cost is included explicitly.

\begin{figure}
\centering
\includegraphics[width=\linewidth]{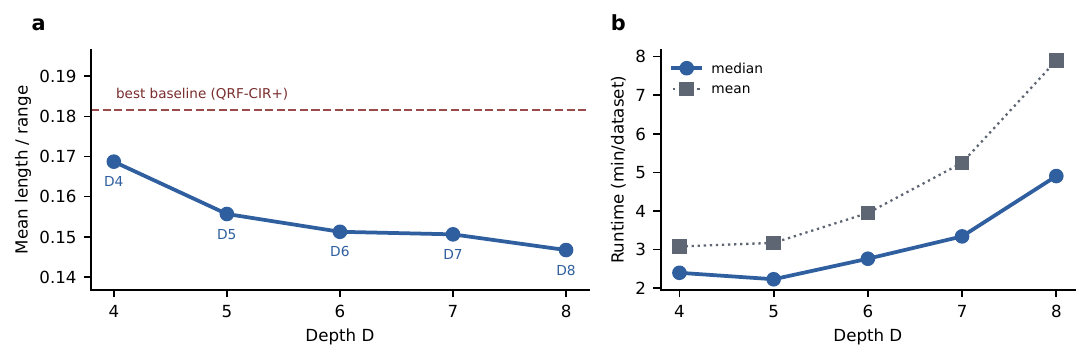}
\caption{Supplementary representative-split tradeoff between depth and runtime for \method at validation target \(\eta=0.905\). (a) Mean normalized interval length by depth, with the dashed line marking the best non-\method baseline in the representative-split comparison pool. (b) Median and mean runtime per dataset.}
\label{fig:depth-runtime-tradeoff-supp}
\end{figure}

Figure~\ref{fig:depth-runtime-tradeoff-supp} shows how increasing the tree depth shortens the extracted intervals while increasing computation. Moving from \(D=4\) to \(D=8\) reaches the most compact setting, with the additional runtime providing finer interval resolution.

\subsection{Strict-Coverage and Paired Bootstrap Analysis}
\label{app:sensitivity-bootstrap}

We evaluate interval efficiency at the nominal \(0.900\) coverage requirement over all 560 dataset--split evaluations. Methods reaching the target are ranked by normalized interval length within each evaluation. Remaining methods follow by coverage error and then length. Table~\ref{tab:strict-threshold-sensitivity-current} reports strict-valid counts, aggregate metrics, ranks, and wins for \method and six strong baselines.

\begin{table}
\centering
\caption{Strict-coverage comparison over 560 evaluations. Methods reaching coverage \(0.900\) are ranked by normalized interval length; the remaining methods follow by coverage error and length. Wins count first-place finishes under this rule.}
\label{tab:strict-threshold-sensitivity-current}
\footnotesize
\setlength{\tabcolsep}{4.0pt}
\begin{tabular*}{0.88\linewidth}{@{\extracolsep{\fill}}lrrrrr@{}}
\toprule
Method & strict & len/rng \(\downarrow\) & score/rng \(\downarrow\) & rank \(\downarrow\) & wins \(\uparrow\) \\
\midrule
\textbf{\method} & 429/560 & \textbf{0.1473} & 0.2627 & \textbf{6.46} & \textbf{177} \\
QRF-CTI & 396/560 & 0.2017 & 0.2578 & 9.36 & 121 \\
QRF-CQR & 417/560 & 0.2159 & 0.2914 & 11.45 & 11 \\
LightGBM-CQR & 349/560 & 0.2002 & 0.2698 & 11.77 & 14 \\
CatBoost-CQR & 313/560 & 0.1786 & 0.2571 & 12.37 & 14 \\
QRF-CIR & 273/560 & 0.1800 & 0.3243 & 13.35 & 13 \\
QRF-CIR+ & 216/560 & 0.1760 & 0.3264 & 15.62 & 16 \\
\bottomrule
\end{tabular*}
\end{table}

Table~\ref{tab:strict-threshold-sensitivity-current} shows that \method satisfies the strict target on 429/560 evaluations, exceeding QRF-CQR by 12 evaluations and QRF-CTI by 33. It simultaneously achieves the shortest mean length/range (\(0.1473\)), best mean rank (\(6.46\)), and most wins (177). The main advantage therefore persists when the nominal coverage requirement is enforced directly rather than through the \(0.885\) tolerance.

We perform paired bootstrap over datasets for the mean normalized length difference between each strong baseline and \method. For each dataset, we first average length/range over the 10 random seeds, then resample the 56 datasets with replacement and calculate the mean difference for each resample. Positive values in Table~\ref{tab:paired-bootstrap-current} indicate shorter \method intervals. The table reports \(95\%\) percentile bootstrap intervals and two-sided exact sign tests over dataset-level comparisons.

\begin{table}
\centering
\caption{Paired dataset bootstrap for mean normalized length gain. Gain is baseline length/range minus \method length/range; positive values indicate shorter \method intervals.}
\label{tab:paired-bootstrap-current}
\footnotesize
\setlength{\tabcolsep}{4.0pt}
\begin{tabular*}{0.82\linewidth}{@{\extracolsep{\fill}}lrrrr@{}}
\toprule
Baseline & gain \(\uparrow\) & 95\% CI & shorter datasets \(\uparrow\) & sign \(p\) \(\downarrow\) \\
\midrule
QRF-CTI & 0.0544 & [0.0156, 0.1005] & 36/56 & 0.044 \\
CatBoost-CQR & 0.0313 & [0.0192, 0.0453] & 46/56 & \(1.2{\times}10^{-6}\) \\
QRF-CIR+ & 0.0287 & [0.0147, 0.0467] & 42/56 & \(2.3{\times}10^{-4}\) \\
QRF-CIR & 0.0327 & [0.0173, 0.0523] & 47/56 & \(2.6{\times}10^{-7}\) \\
LightGBM-CQR & 0.0529 & [0.0288, 0.0848] & 49/56 & \(<10^{-8}\) \\
QRF-CQR & 0.0686 & [0.0480, 0.0920] & 54/56 & \(<10^{-8}\) \\
\bottomrule
\end{tabular*}
\end{table}

All bootstrap intervals in Table~\ref{tab:paired-bootstrap-current} are strictly positive, and every exact sign test rejects equality at the \(0.05\) level. Against QRF-CTI, the closest competitor by aggregate rank, \method gains \(0.0544\) mean length/range with a \(95\%\) interval of \([0.0156,0.1005]\). The gain is positive on 36/56 datasets despite this being the most difficult pairwise comparison. Against the remaining strong baselines, \method is shorter on 42--54 datasets with substantially smaller \(p\)-values. These paired results confirm that the aggregate compactness advantage is supported across datasets.

\begin{table}
\centering
\caption{Representative-split dataset-level coverage diagnostics for the main \(D=8,\eta=0.905\) setting.}
\label{tab:coverage-diagnostics-current}
\footnotesize
\setlength{\tabcolsep}{3.0pt}
\begin{tabular*}{\linewidth}{@{\extracolsep{\fill}}lrrrrrrrrr@{}}
\toprule
Dataset & cov & val cov & gap \(\downarrow\) & len/rng \(\downarrow\) & lower miss \(\downarrow\) & upper miss \(\downarrow\) & \(M\) & \(T\) & rank \(\downarrow\) \\
\midrule
laser & 0.8040 & 0.9095 & 0.1055 & 0.0574 & 0.1407 & 0.0553 & 0.55 & 1.20 & 23 \\
forest\_fires & 0.8462 & 0.9135 & 0.0673 & 0.0270 & 0.0000 & 0.1538 & 0.60 & 1.50 & 18 \\
meta & 0.8679 & 0.9057 & 0.0377 & 0.0055 & 0.0189 & 0.1132 & 0.70 & 0.70 & 22 \\
Titanic & 0.8740 & 0.9084 & 0.0344 & 0.0793 & 0.0573 & 0.0687 & 0.90 & 0.90 & 21 \\
socmob & 0.8836 & 0.9351 & 0.0514 & 0.0586 & 0.0086 & 0.1078 & 0.65 & 1.50 & 24 \\
\bottomrule
\end{tabular*}
\end{table}

Table~\ref{tab:coverage-diagnostics-current} isolates the five representative-split datasets below the \(0.885\) threshold; the remaining 51/56 datasets are valid. The failures coincide with validation-to-test coverage losses of \(0.0344\)--\(0.1055\) and, in four cases, strongly asymmetric tail misses. Their normalized lengths remain small, showing that the issue is selection transfer rather than uniformly wide or degenerate intervals. The analysis localizes the principal coverage failures to a small set of distribution-shifted datasets.

\subsection{Dataset Factors, Conditional Coverage, and Runtime}
\label{app:dataset-factor-runtime}

\begin{table}
\centering
\caption{Representative-split dataset factor summary for the main \method setting. Strata are quartiles by total sample size and by the ratio of unique target values.}
\label{tab:dataset-factor-summary-current}
\footnotesize
\setlength{\tabcolsep}{4.0pt}
\begin{tabular*}{0.92\linewidth}{@{\extracolsep{\fill}}lrrrrrr@{}}
\toprule
Stratum & datasets & mean rank \(\downarrow\) & wins \(\uparrow\) & valid & len/rng \(\downarrow\) & mean cov \\
\midrule
sample=small & 14 & 10.50 & 1 & 9/14 & 0.1396 & 0.9037 \\
sample=medium & 14 & 4.57 & 5 & 14/14 & 0.2289 & 0.9129 \\
sample=large & 14 & 5.36 & 6 & 14/14 & 0.1174 & 0.9118 \\
sample=xlarge & 14 & 3.00 & 3 & 14/14 & 0.1009 & 0.9282 \\
unique=low & 14 & 5.64 & 1 & 14/14 & 0.0746 & 0.9388 \\
unique=mid-low & 14 & 3.07 & 6 & 14/14 & 0.1685 & 0.9037 \\
unique=mid-high & 14 & 7.86 & 3 & 10/14 & 0.1315 & 0.8899 \\
unique=high & 14 & 6.86 & 5 & 13/14 & 0.2122 & 0.9242 \\
\bottomrule
\end{tabular*}
\end{table}

Table~\ref{tab:dataset-factor-summary-current} summarizes these relationships by sample-size and target-uniqueness quartiles. All 42 datasets in the medium, large, and extra-large sample strata satisfy the coverage threshold. The extra-large stratum achieves mean rank \(3.00\) and mean length/range \(0.1009\). The low and mid-low target-uniqueness strata are also valid on all 28 datasets. Coverage failures concentrate in the smallest sample and mid-high uniqueness groups, indicating where validation selection is least stable. Overall, \method is strongest when sufficient data support interval selection, while remaining compact across every sample-size stratum.

\begin{table}
\centering
\caption{Representative-split supplementary diagnostics for all compared methods. Gaps are macro averages over datasets; runtime is median total seconds per dataset.}
\label{tab:diagnostic-metrics-current}
\footnotesize
\setlength{\tabcolsep}{3.0pt}
\begin{tabular*}{\linewidth}{@{\extracolsep{\fill}}lrrrrrr@{}}
\toprule
Method & \shortstack{lower miss\\\(\downarrow\)} & \shortstack{upper miss\\\(\downarrow\)} & \shortstack{val/test gap\\\(\downarrow\)} & \shortstack{target quartile gap\\\(\downarrow\)} & \shortstack{length tertile gap\\\(\downarrow\)} & sec \(\downarrow\) \\
\midrule
\textbf{\method} & 0.0273 & 0.0586 & 0.0113 & 0.1802 & 0.0507 & 294.2 \\
QRF-CTI & 0.0094 & 0.0196 & 0.0107 & 0.2040 & 0.0897 & 88.8 \\
QRF-CHR & 0.0238 & 0.0669 & 0.0127 & 0.1957 & 0.0511 & 358.4 \\
QRF-CIR+ & 0.0290 & 0.0710 & 0.0134 & 0.2051 & 0.0574 & 87.2 \\
QRF-CIR & 0.0296 & 0.0674 & 0.0134 & 0.1996 & 0.0587 & 86.4 \\
XGBoost-CQR & 0.0382 & 0.0548 & 0.0107 & 0.1785 & 0.0872 & 1304.2 \\
LightGBM-CQR & 0.0365 & 0.0537 & 0.0122 & 0.1652 & 0.0760 & 1083.5 \\
CatBoost-CQR & 0.0422 & 0.0513 & 0.0118 & 0.1408 & 0.0606 & 5.3 \\
CatBoost-SplitCP & 0.0395 & 0.0539 & 0.0128 & 0.2246 & 0.0853 & 3.1 \\
MLP-CTI & 0.0099 & 0.0583 & 0.0116 & 0.2729 & 0.1249 & 233.3 \\
QRF-CQR & 0.0359 & 0.0416 & 0.0112 & 0.1465 & 0.0560 & 83.4 \\
XGBoost-SplitCP & 0.0456 & 0.0523 & 0.0132 & 0.2364 & 0.0577 & 895.4 \\
LightGBM-SplitCP & 0.0453 & 0.0507 & 0.0115 & 0.2207 & 0.0896 & 1654.2 \\
RF-SplitCP & 0.0397 & 0.0563 & 0.0118 & 0.2388 & 0.0551 & 448.1 \\
RF-A-LOCART & 0.0321 & 0.0592 & 0.0079 & 0.1990 & 0.0406 & 77.8 \\
MLP-R2CCP & 0.0835 & 0.0505 & 0.0112 & 0.2938 & 0.1433 & 984.1 \\
MLP-CHR & 0.0485 & 0.0738 & 0.0108 & 0.2582 & 0.0901 & 506.7 \\
RF-TQ-CQR & 0.0386 & 0.0349 & 0.0120 & 0.1747 & 0.0654 & 699.6 \\
RF-A-LOFOREST & 0.0439 & 0.0575 & 0.0090 & 0.2058 & 0.0509 & 78.9 \\
MLP-CIR+ & 0.0228 & 0.1154 & 0.0108 & 0.2699 & 0.1178 & 232.0 \\
RF-LOCART & 0.0322 & 0.0585 & 0.0090 & 0.1982 & 0.0399 & 80.2 \\
MLP-SplitCP & 0.0339 & 0.0612 & 0.0135 & 0.2460 & 0.0531 & 122.4 \\
MLP-CIR & 0.0229 & 0.0663 & 0.0102 & 0.2231 & 0.0914 & 233.9 \\
RF-LOFOREST & 0.0595 & 0.0562 & 0.0127 & 0.2126 & 0.0525 & 78.9 \\
MLP-CQR & 0.0402 & 0.0547 & 0.0101 & 0.2091 & 0.0759 & 231.4 \\
\bottomrule
\end{tabular*}
\end{table}

Table~\ref{tab:diagnostic-metrics-current} compares conditional coverage and runtime across every method on the representative split. The \method validation-to-test gap is \(0.0113\), comparable to the strongest baselines. Its target-quartile gap, \(0.1802\), improves on QRF-CTI, QRF-CHR, QRF-CIR, and QRF-CIR+, while its length-tertile gap \(0.0507\) remains among the lower values. Runtime is higher than the fastest tree methods but below several neural, boosted-tree, and TreeQuantile alternatives. These diagnostics show that the shortest intervals are obtained with competitive conditional coverage variation and moderate computation.

\section{Per-Dataset Details}
\label{app:per-dataset-results}

Table~\ref{tab:per-dataset-detailed-results} provides frozen records for every method and dataset on the representative split. It exposes coverage, normalized length, validity, and rank under the same rule as the main comparison; coverage error is omitted because it is determined directly by coverage and the nominal target. The table therefore permits direct verification of both \method wins and datasets where strong baselines remain competitive.

\begingroup
\footnotesize
\setlength{\tabcolsep}{2.5pt}
\renewcommand{\arraystretch}{1.03}
\begin{longtable}{@{}L{0.34\textwidth}L{0.28\textwidth}C{0.08\textwidth}C{0.09\textwidth}C{0.07\textwidth}C{0.06\textwidth}@{}}
\caption{Representative-split per-dataset test records for the 56-dataset comparison pool. PATH uses \(D=8\) and validation selection target \(0.905\). Valid indicates empirical coverage at least \(0.885\); rank follows the coverage-based ranking rule.}
\label{tab:per-dataset-detailed-results}\\
\toprule
Dataset & Method & cov & len/rng & valid & rank \\
\midrule
\endfirsthead
\multicolumn{6}{c}{\footnotesize Table~\thetable: Per-dataset records, continued.}\\
\toprule
Dataset & Method & cov & len/rng & valid & rank \\
\midrule
\endhead
\midrule
\multicolumn{6}{r}{\footnotesize Continued on next page}\\
\endfoot
\bottomrule
\endlastfoot
Allstate\_Claims\_Severity & MLP-CTI & 0.8994 & \textbf{0.0377} & yes & \textbf{1} \\
Allstate\_Claims\_Severity & \method & 0.9109 & 0.0384 & yes & 2 \\
Allstate\_Claims\_Severity & LightGBM-CQR & \textbf{0.9001} & 0.0387 & yes & 3 \\
Allstate\_Claims\_Severity & QRF-CIR+ & 0.8901 & 0.0389 & yes & 4 \\
Allstate\_Claims\_Severity & XGBoost-CQR & 0.8994 & 0.0390 & yes & 5 \\
Allstate\_Claims\_Severity & QRF-CHR & 0.8988 & 0.0394 & yes & 6 \\
Allstate\_Claims\_Severity & QRF-CIR & 0.8969 & 0.0396 & yes & 7 \\
Allstate\_Claims\_Severity & CatBoost-CQR & 0.8991 & 0.0402 & yes & 8 \\
\end{longtable}
\endgroup

Table~\ref{tab:per-dataset-detailed-results} shows that \method is valid on 51/56 datasets and ranks first on 15 datasets in the representative split, the largest win count in that comparison. Its wins span datasets with markedly different target scales, sample sizes, and feature structures. The remaining rows identify the settings where QRF-CTI and other strong baselines are competitive. These records confirm that the aggregate advantage is distributed across heterogeneous regression tasks rather than driven by a small set of datasets.

\end{document}